\documentclass{article}

\usepackage{microtype}
\usepackage{graphicx}
\usepackage{float}
\usepackage{subcaption}
\usepackage{booktabs}
\usepackage{hyperref}

\usepackage[accepted]{icml2026}


\usepackage{amsmath}
\usepackage{amssymb}
\usepackage{mathtools}
\usepackage{amsthm}
\usepackage{enumitem}

\usepackage[capitalize,noabbrev]{cleveref}

\theoremstyle{plain}
\newtheorem{theorem}{Theorem}[section]
\newtheorem{proposition}[theorem]{Proposition}

\theoremstyle{definition}

\theoremstyle{remark}

\newcommand{\reals}{\mathbb{R}}

\newcommand{\choice}{\texttt{Classical.choice}}
\newcommand{\propext}{\texttt{propext}}
\newcommand{\quotsound}{\texttt{Quot.sound}}

\icmltitlerunning{Geometric Measurements of the Axiom of Choice in Neural Proof Embeddings}

\begin{document}

\twocolumn[
  \icmltitle{
      Geometric Measurements of the Axiom of Choice in Neural Proof Embeddings}

  \icmlsetsymbol{equal}{*}

  \begin{icmlauthorlist}
    \icmlauthor{Rodrigo Mendoza-Smith}{anon}
  \end{icmlauthorlist}

  \icmlaffiliation{anon}{Independent Researcher}

  \icmlcorrespondingauthor{Rodrigo Mendoza-Smith}{rms@isotropic.sh}

  \icmlkeywords{theorem proving, formalized mathematics, Lean, Mathlib,
    axiom of choice, proof embeddings, neural theorem proving, out-of-distribution
    generalization}

  \vskip 0.3in
]

\printAffiliationsAndNotice{}

\begin{abstract}
The axiom of choice has divided the foundations of mathematics for over a century, but the distinction between classical and constructive proofs has remained a philosophical and methodological one. 
We use Lean~4's kernel-level tracking of axiom dependence to show that the axiom of choice has a measurable geometric correlate in proof space that obeys a one-parameter mixture law and has operational consequences for neural theorem provers.
To do this, we partition $471{,}260$ declarations of Mathlib by transitive dependence on the axiom of choice and represent a filtered population of $42{,}355$ traced theorems by their sequences of tactic invocations.
We use the constructive proofs in this dataset to train a self-supervised proof encoder and show that when using it to measure classical proofs, three complementary measurements~(anomaly score, reconstruction loss, and density-superlevel containment) exhibit a common decline with the proof's distance from the axiom in the dependency graph, from sharp separation at the shallow boundary (AUC $0.847$ at distance~$2$) to indistinguishability at distance~$9{+}$. 
Robustness controls show that the signature survives length, file, author, and topic controls, and replicates under full-source encoders trained on normalised proof source.
Operationally, we show that on an evaluation sample of $251$ Mathlib theorems, Lean's \texttt{aesop} tactic solves constructive theorems at $13\times$ the rate of classical ones, and a neural-guided hybrid using the ReProver tactic generator compresses the gap to $5\times$.
The geometric anomaly score predicts \texttt{aesop} failure beyond proof length, providing an operational link between the geometric signature and prover performance.
\end{abstract}


The axiom of choice has divided the foundations of mathematics for more than a century.
Since Zermelo's formulation \citep{zermelo1908investigations}, constructive mathematicians from Brouwer to \citet{bishop1967foundations} have held that a proof of existence ought to produce a witness, while classical mathematicians have built with the non-constructive shortcuts that the axiom permits.
The disagreement has remained philosophical and methodological, and the question of whether these two dogmas differ structurally in ways that go beyond mere axiom dependence has eluded mathematical practice itself.
In this paper, we argue that the axiom of choice can be studied geometrically and that its correlates are operationally consequential for neural theorem provers.

Recent progress in machine learning for mathematical reasoning \citep{polu2022formal, yang2023leandojo, hubert2025olympiad} has driven rapid growth in interactive theorem provers, of which Lean~4 \citep{demoura2021lean} and its library Mathlib \citep{mathlib2020lean} are now the largest single instance.
In particular, Mathlib contains $471{,}260$ machine-certified declarations spanning algebra, analysis, topology, number theory, and other subfields.
These declarations are written in a single dependent type theory and verified by Lean's kernel.
Each of these declarations can be represented geometrically through their embeddings computed from self-supervised methods, or structurally through their relationship to other mathematical objects in the library.
This makes Mathlib an unusually rich dataset for the study of meta-mathematics where objects are naturally equipped with both a geometric and ontological representation.
We exploit this dual structure to ask whether the axiom of choice has an empirical signature in the proofs themselves.
To do this, we use the Lean kernel to track axiom dependence across Mathlib and compute a stratification of the library by their {\em depth} or distance from the axiom of choice, and a partition of its dependency graph into a set of {\em classical proofs} that transitively depend on the axiom of choice, and a set of {\em constructive proofs} that do not.

We use this data to train a self-supervised denoising encoder on tactic sequences from the constructive side of the partition and use it to project Mathlib's classical population into its embedding space.
Three mathematically distinct measurements~(a $k$-nearest-neighbour anomaly score against constructive embeddings, the encoder's own masked-token reconstruction loss, and a density-superlevel containment fraction inside the constructive density's superlevel sets) each separate classical from constructive proofs sharply at the shallow boundary ($k$-NN AUC $0.847$, reconstruction-loss excess $34.5\%$, containment fraction $43\%$ outside the constructive $90\%$ region) and decay to the constructive baseline by depth nine.

We model these {\em depth gradients} as arising from a mixture of two populations at each depth $d$: a constructive proof distribution $P$, and a distribution $R$ of proofs that directly invoke the missing classical capability.
Further, we argue that if we let $\lambda_d \in [0, 1]$ be the fraction of directly-classical mass at depth $d$, the depth-$d$ proof distribution is
\begin{equation}
    \label{eq:mixture}
  Q_d \;=\; (1 - \lambda_d)\, P \;+\; \lambda_d\, R,
\end{equation}
with $\lambda_2 \geq \lambda_3 \geq \cdots \geq 0$ reflecting that direct invocations of the axiom concentrate at shallow depths.
The depth law that emerges from the experiments is therefore not three separate empirical findings but a single one-parameter law over distance, in which this foundational distinction is summarised by one mixing weight per dependency-distance bucket.

These geometric findings have direct consequences for neural theorem provers, which learn to compose tactics by training on libraries such as Mathlib.
On a held-out sample of $251$ Mathlib theorems, Lean's \texttt{aesop} automation \citep{limperg2023aesop}, a symbolic best-first search over a curated rule set, solves $20\%$ of constructive theorems and only $1.5\%$ of classical ones, a gap of roughly $13\times$.
Moreover, pairing \texttt{aesop} with the ReProver tactic generator \citep{yang2023leandojo} compresses the gap to roughly $5\times$ but does not close it.
The geometric anomaly score from the constructive encoder predicts prover failure beyond what proof length alone explains, which suggests that the depth law in proof-embedding geometry is operationally relevant for neural theorem proving.

\begin{figure*}[!t]
  \centering
  \includegraphics[width=0.96\textwidth]{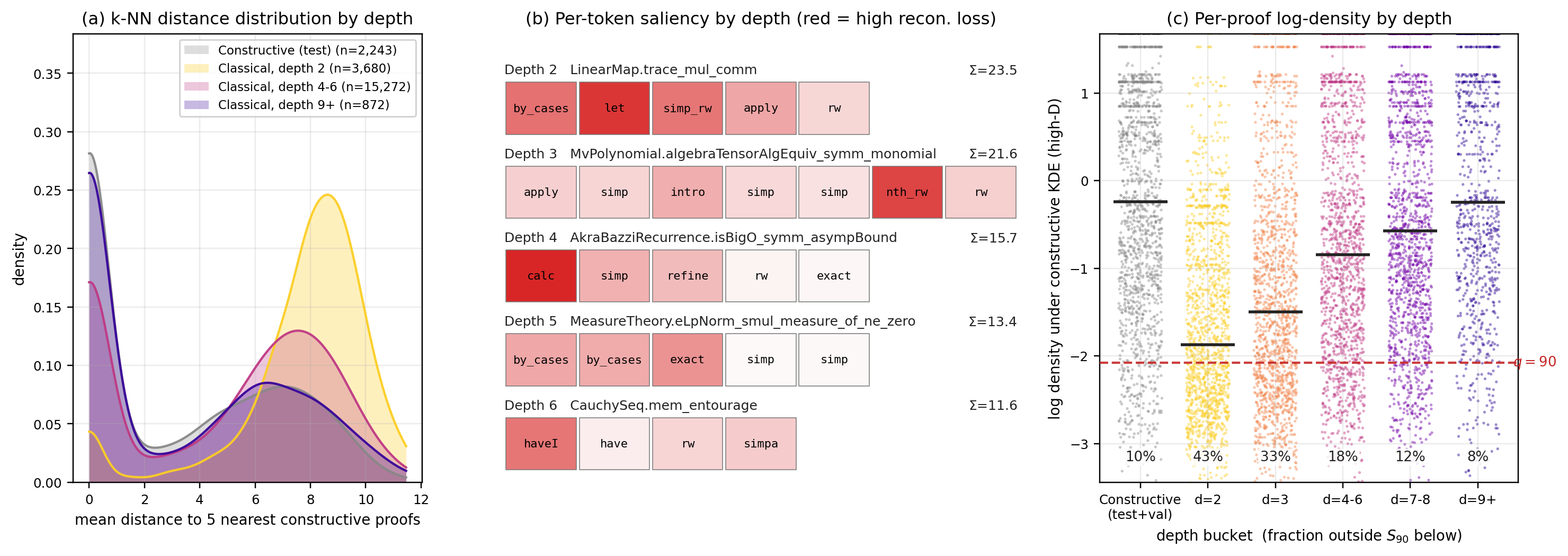}
  \caption{\textbf{The depth law in three views.} (a)
  \emph{$k$-NN distance distribution by depth}: histograms of mean
  Euclidean distance to the $5$ nearest constructive-training proofs
  (constructive-standardised encoder embeddings; \Cref{app:models}).
  (b) \emph{Per-token saliency}: one representative classical proof per
  depth bucket; cell colour intensity is the single-token reconstruction
  loss under the constructive encoder, and $\Sigma$ is the summed
  per-token loss.
  (c) \emph{Per-proof log-density by depth}: each dot is one proof,
  $y$ is its log-density under the high-D constructive KDE that
  defines the $S_{90}$ superlevel set, $x$ is its depth bucket.
  Black ticks are bucket medians; the red dashed line is the q=90
  threshold; the percentage below each column is the fraction of
  proofs in that bucket falling outside $S_{90}$.}
  \label{fig:hero}
\end{figure*}

\paragraph{Related work.}
Machine learning for interactive theorem proving has moved from tactic prediction in proof assistants \citep{huang2018gamepad,yang2019learning} to language-model and expert-iteration provers \citep{polu2020generative,han2021proof,polu2022formal}, retrieval and premise selection systems \citep{yang2023leandojo,mikula2024magnushammer}, and reinforcement-learning provers such as AlphaProof \citep{hubert2025olympiad}.
These works primarily aim to improve proof search or to evaluate generalisation across formal libraries; for example, LeanDojo's novel-premises split tests whether a prover can use premises not seen during training.
Our question is orthogonal: we use Lean's kernel-level dependence on \choice{} to define a proof-structural axis, and ask whether that axis is visible in the geometry of a representation trained only on proofs that don't depend on \choice{}.
Methodologically, this places our measurements near one-class and out-of-distribution detection \citep{scholkopf2001estimating,hendrycks2016baseline,liang2017enhancing}, but with the held-out population defined by logical dependence rather than dataset origin.
More loosely, the paper is motivated by proof complexity's study of how proof systems and inference principles affect proof length \citep{cook1979relative,krajicek1995bounded,razborov2003propositional}, by the classical/constructive fault line around choice \citep{zermelo1904wellordering,bishop1967foundations}, and by the older theme that formal symbol systems can exhibit global structure not apparent from their local rules \citep{hofstadter1979geb}.

\section{Setup}
\label{sec:setup}

\paragraph{A partition on Mathlib. }The Lean \emph{kernel} is the type-checker for a fixed dependent type theory \citep{demoura2021lean}.
It accepts as input a term and a type, both expressed in the kernel's internal syntax, and either certifies that the term inhabits the type or rejects the input.
More precisely, it accepts a proof if and only if every step reduces to a finite chain of inferences that bottom out in a fixed list of foundational axioms.
Lean's kernel tracks dependence on three foundational axioms: (i) \propext{}, which asserts that two propositions are equal if they are logically equivalent; (ii) \quotsound{}, which asserts that if two elements are related by an equivalence relation, their equivalence classes are equal; and (iii) \choice{}, which asserts that for any nonempty type, an element can be extracted.
We use this mechanism to partition Mathlib's $471{,}260$ declarations.
Performing breadth-first search through the dependency graph, in which each edge $u \to v$ records that declaration $u$'s proof invokes declaration $v$, we partition Mathlib into a set of \emph{classical} proofs that transitively depend on \choice{} and a set of \emph{constructive} proofs that do not.
Here and throughout, ``classical'' and ``constructive'' refer to the recorded Mathlib proof, not to whether the theorem statement admits an alternative proof without \choice{}.
The dependency graph is acyclic because proofs can only invoke previously-defined declarations, so the dependency chain from any classical declaration back to \choice{} has a well-defined shortest length.
We use this length as a distance metric: the \emph{distance} of a classical declaration from \choice{} is the length of the shortest dependency chain from that declaration to \choice{} itself.

The full kernel dependency graph contains $471{,}260$ Mathlib declarations, of which $171{,}522$ are classical and $299{,}738$ are constructive.
For our learned proof-representation experiments, we restrict to LeanDojo-traced theorems \citep{yang2023leandojo} having between $2$ and $200$ tactic invocations each with an extractable leading tactic head.
This leaves a final proof-embedding population of $42{,}355$ theorems: $31{,}144$ classical and $11{,}211$ constructive, with median tactic-trace length $4$.
The full chain is reported in \Cref{tab:counts}.

\subsection{Proof Representations}
\label{sec:setup:representations}

Every proof in our population is recorded as a sequence of tactic invocations.
We summarise each invocation by its \emph{tactic head}, the first identifier of the call (for example, \texttt{intro}, \texttt{rw}, or \texttt{simp}), and represent the proof as the resulting sequence of heads.
This is a coarse but informative view: the head says which kind of move was applied at each step, even if it discards the specific arguments.
Constructive tactic sequences in our population have median length $7$ and $95$th percentile $25$, which fit comfortably within a fixed window of $64$ steps.
The vocabulary contains $219$ distinct heads that occur at least five times in constructive proofs, together with a handful of special tokens.

\begin{figure*}[!t]
  \centering
  \includegraphics[width=0.98\textwidth]{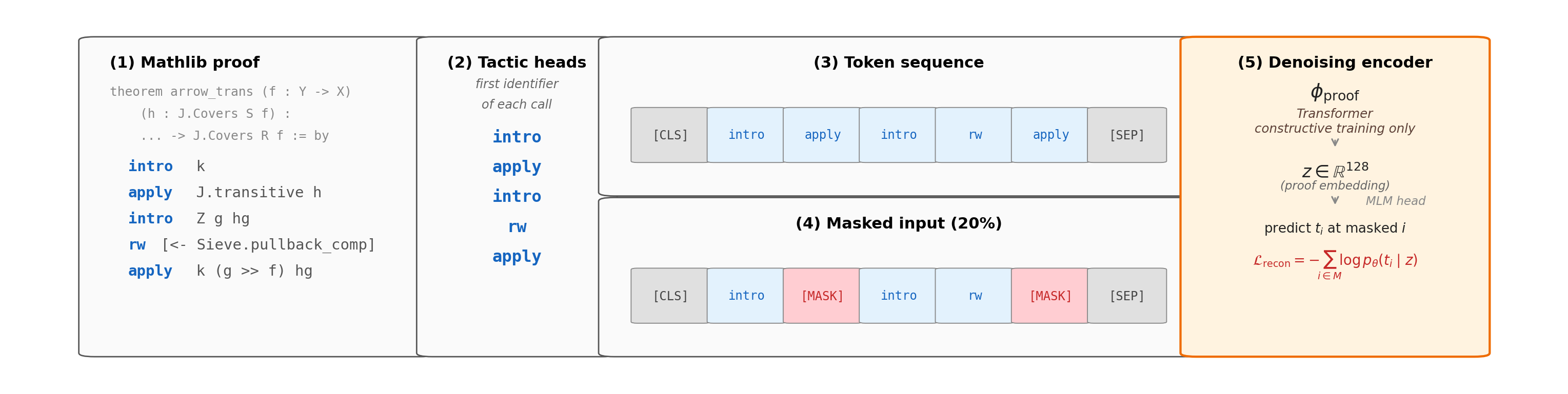}
  \caption{\textbf{Proof representation and denoising pipeline.} Five stages, left to right, illustrated on a real Mathlib proof of \texttt{GrothendieckTopology.arrow\_trans} (\texttt{<-} and \texttt{>>} in the body stand for the Lean unicode \texttt{$\leftarrow$} and \texttt{$\gg$}).
(1) A tactic-mode proof body, with each line a tactic invocation and the leading identifier in bold blue.
(2) The \emph{tactic head} of each invocation is its first identifier, e.g.\ \texttt{rw} for \texttt{rw [<- Sieve.pullback\_comp]}; we keep this and discard the arguments.
(3) The proof becomes an ordered sequence of heads bracketed by \texttt{[CLS]} and \texttt{[SEP]}.
(4) During training, $20\%$ of head tokens are replaced by \texttt{[MASK]} (red), with special tokens never masked.
(5) A Transformer denoising encoder $\phi_{\rm proof}$, trained only on constructive proofs, produces a $128$-dimensional embedding from the masked sequence, and a paired decoder predicts the masked tokens.
All downstream measurements use this same frozen encoder.}
  \label{fig:proof_representation}
\end{figure*}

To turn these sequences into a geometry, we train a \emph{denoising autoencoder} on constructive proofs alone.
A denoising autoencoder is a sequence-to-sequence model that learns to reconstruct its input from a corrupted version of itself, in the same spirit as the masked language modelling objective used to pretrain BERT \citep{devlin2019bert}.
Our model is a single encoder--decoder Transformer trained end-to-end: a $4$-layer encoder compresses the corrupted sequence into a sequence of hidden states in $\reals^{128}$, and a $2$-layer decoder with a linear vocabulary head predicts the masked tokens from that representation.
At training time, for each tactic-head sequence $x = (x_1, \ldots, x_n)$, we sample a random subset $M \subset \{1, \ldots, n\}$ containing $20\%$ of the non-special positions, replace $x_i$ with a \texttt{[MASK]} token for each $i \in M$ to form a corrupted input $\tilde{x}$, and minimise the cross-entropy of the model's predictions $\hat{p}(\,\cdot\, \mid \tilde{x})$ on the masked positions:
\[
  \mathcal{L}(x)
  \;=\;
  -\frac{1}{|M|}\sum_{i \in M}
  \log \hat{p}(x_i \mid \tilde{x}).
\]
Reconstruction quality improves only by representing the training distribution well, so the encoder learns features that explain how constructive proofs are typically composed without ever being told which proofs are classical. 
At inference time we feed the uncorrupted sequence $x$ to the encoder and take the mean of its hidden states over all non-padding positions,
\[
  \phi_{\rm proof}(x)
  \;=\;
  \frac{1}{|S|}
  \sum_{ i \in S} h_i(x)
  \;\in\; \reals^{128},
\]
where $S = \{i : x_i \neq \texttt{[PAD]}\}$ and $h_i(x)$ is the $i$-th encoder hidden state; \texttt{[CLS]} and \texttt{[SEP]} positions are retained in the pool while \texttt{[PAD]} positions are excluded.
We refer to $\phi_{\rm proof}(x)$ as the \emph{proof embedding} (\Cref{fig:proof_representation}).
Before all downstream analyses we $L^2$-normalise each embedding, so the saved representations live on the unit sphere $\mathbb{S}^{127} \subset \reals^{128}$; full downstream detector preprocessing is described in \Cref{app:models}.
The decoder and vocabulary head are used only at training time and discarded for all downstream analyses; only $\phi_{\rm proof}$ enters the boundary measurements that follow.
More generally, the encoder is \emph{label-free}: no information about the classical/constructive partition influences its training data, vocabulary, or hyperparameters.
Therefore, when classical proofs enter the analysis, they enter as a held-out population to be measured, never as training signal or as a selection criterion.
Details about our label-free methodology and our architectural and training decisions are given in \Cref{app:models}.


\section{The Depth Law}
\label{sec:law}

We begin with the aggregate question: how well can an anomaly score distinguish classical from constructive proofs?
To test this, we fit a $k$-NN ($k=5$) detector on a sample of $8{,}968$ constructive training proofs and compute an \emph{anomaly score} on a held-out test split of $1{,}122$ constructive-test proofs and a fixed random subsample of $5{,}000$ classical proofs drawn from the $31{,}144$ classical theorem proof-embedding population.
Proofs near the constructive cluster receive low scores, while proofs far from it receive high scores.
We summarise the score's ability to distinguish two populations by the area under the receiver operating characteristic curve (AUC).
The detector reaches AUC $0.672$. Length residualisation (see \Cref{app:models}) reaches $0.675$ AUC.  
This suggests that the geometric signal of axiom dependence is real but, at the aggregate level, moderate.
\begin{table}[h!]
\centering
\caption{$k$-NN ($k{=}5$), Gaussian KDE ($h = 0.3$, $5$-fold CV-selected),
and Isolation Forest, on Lean proof embeddings, stratified by distance to \choice{}.
The per-depth rows use the full available population in each displayed bucket.
The two aggregate rows use the same fixed $5{,}000$-classical subsample used in \Cref{tab:lean_support_full}. The $30$ direct depth-$1$ uses of \choice{} are omitted from the per-depth rows.}
\label{tab:depth_strat}
\small
\begin{tabular}{lrccc}
\toprule
\textbf{Bucket} & $n$ & $k$-NN & KDE & IsoForest \\
\midrule
depth $2$     & $3{,}680$ & $\mathbf{0.847}$ & $\mathbf{0.801}$ & $0.782$  \\
depth $3$     & $8{,}686$ & $0.736$          & $0.710$          & $0.699$     \\
depth $4$     & $8{,}299$ & $0.648$          & $0.636$          & $0.607$     \\
depth $5$     & $4{,}393$ & $0.598$          & $0.593$          & $0.562$     \\
depth $6$     & $2{,}580$ & $0.600$          & $0.594$          & $0.565$     \\
depth $7{-}8$ & $2{,}604$ & $0.576$          & $0.567$          & $0.533$     \\
depth $9{+}$  & $872$     & $0.507$          & $0.495$          & $0.463$     \\
\midrule
\textbf{Aggregate AUC} &     &          &           &      \\
\midrule
Raw          & $5{,}000$ & $0.672$ & $0.653$ & $0.636$ \\
Length residualised & $5{,}000$ & $0.675$ & $0.661$ & $0.628$ \\
\bottomrule
\end{tabular}
\end{table}

This aggregate hides a sharper phenomenon: not all classical theorems are equally classical.
A proof one hop from \choice{} invokes the axiom through a single classical tactic, while a proof ten hops away is classical only because some lemma in its derivation eventually invokes one that depends on the axiom.
For the depth-conditional rows, \Cref{tab:depth_strat} uses all $31{,}114$ classical proof embeddings at depth $2$ or greater and reports the AUC for $k$-NN, KDE \citep{silverman1986density}, and Isolation Forest \citep{liu2008isolation} on the encoder embeddings.
At depth $2$ all detectors separate classical from constructive proofs with AUC between $0.78$ and $0.847$.
Separation declines with depth and at depth $9{+}$, every detector is at or near chance.
The shape is the same across methods and a sliced Wasserstein test (\Cref{sec:robustness}) reproduces the gradient at the level of distributions rather than detection scores.
\Cref{tab:depth_strat} suggests that the right object of study is the gradient, which we name the \emph{depth law}: separation from the constructive distribution decays with distance from \choice{}, with a similar trend across detectors.

The agreement between $k$-NN, KDE, and Isolation Forest shows that the depth gradient is not tied to a particular one-class detector. 
But all three are still external scores placed on top of the frozen embedding space. 
To test whether the same ordering is visible in the learned proof distribution itself, we add two complementary measurements.
Reconstruction loss returns to the denoising objective: a proof is atypical if its masked tactic tokens are hard for the constructive-trained model to predict.
Superlevel containment instead asks a geometric question: whether the proof lies inside the high-density region occupied by constructive embeddings.
Thus the next two tests measure the same depth gradient from two different angles, predictive fit and typical-set membership.

\subsection{Measurement \#2: Reconstruction Loss}
\label{sec:reconstruction_loss}

Recall that $\phi_{\rm proof}$ was trained to recover masked tactic tokens in constructive proofs, with no classical labels involved (\Cref{sec:setup:representations}).
If classical proofs lie outside the constructive training distribution, the same frozen model should incur higher cross-entropy on them.
For each held-out proof $x$ we draw ten independent random masks $M^{(1)}, \ldots, M^{(10)}$ hiding $20\%$ of the non-special tokens, and average per-token cross-entropy on the masked positions:
\[
  \ell(x)
  \;=\;
  \frac{1}{10}\sum_{r=1}^{10}\;
  \frac{1}{|M^{(r)}|}\sum_{i \in M^{(r)}}
  -\log \hat{p}\bigl(x_i \mid \tilde{x}^{(r)}\bigr),
\]
where $\tilde{x}^{(r)}$ is the corruption of $x$ under the $r$-th mask.
The excess loss $\ell(x) - \mathbb{E}_{x \sim \mathrm{constructive\ test}}[\ell(x)]$ is a quantitative generalisation gap in the units the model was trained in, with no threshold step.
\Cref{fig:reconstruction_loss} shows that held-out constructive proofs incur a baseline of $2.47$ nats per masked token, while the shallow classical reconstruction bucket ($d \leq 2$) incurs $3.32$ nats ($+34.5\%$ with Mann-Whitney one-sided $p \approx 3 \times 10^{-101}$).
The excess attenuates with depth and is at the constructive baseline by depth $9{+}$, mirroring the overall AUC gradient.
Residualising against a length-based prediction fit on constructive training proofs preserves the gradient in full.
\begin{figure}[ht]
  \centering
  \includegraphics[width=0.98\columnwidth]{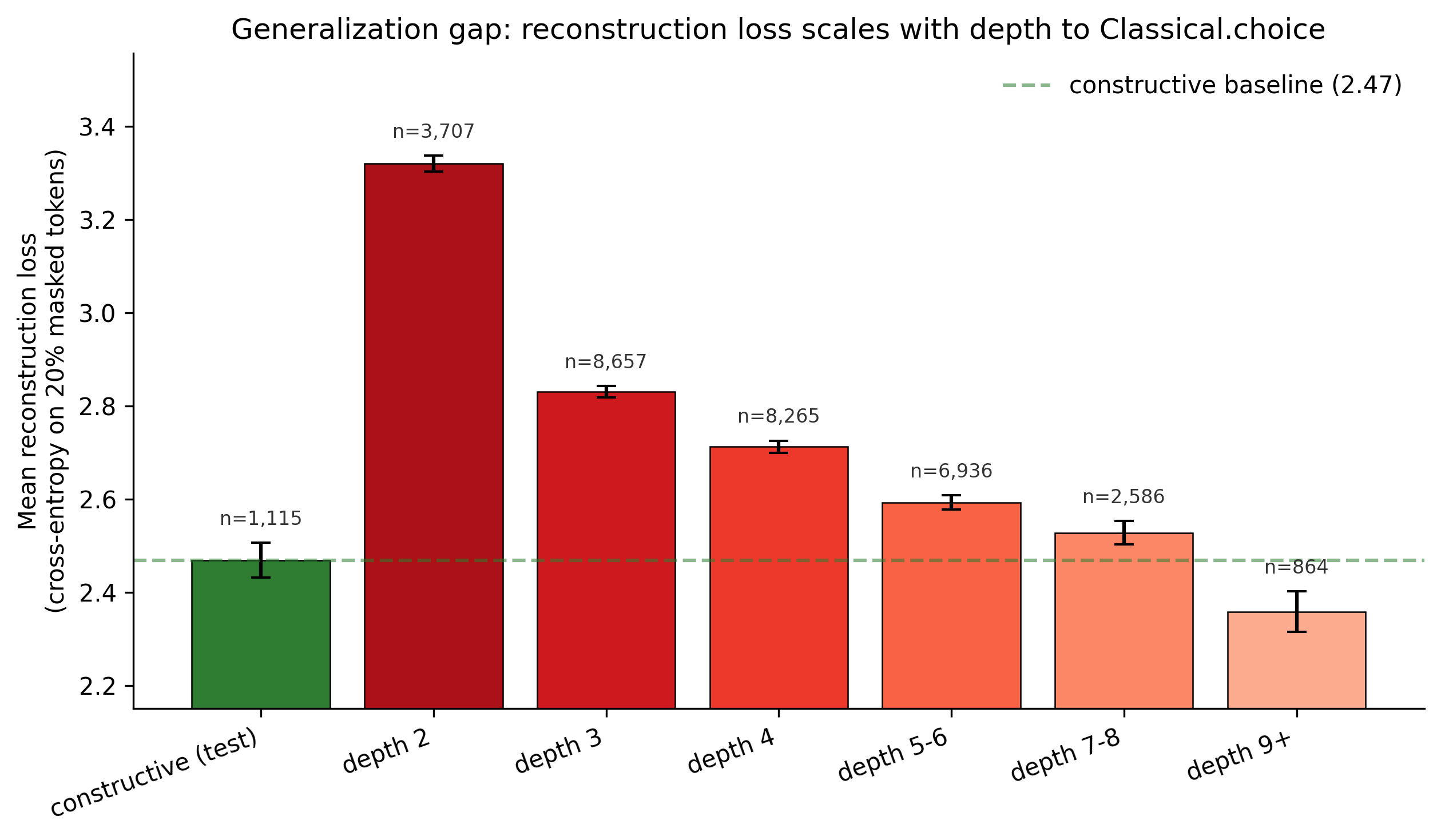}
    \caption{Mean reconstruction loss by depth from \choice{}.
  Green bar: held-out constructive test proofs.
  Red bars: classical proofs stratified by depth; the first bar pools
  the $27$ maskable depth-$1$ proofs with the $3{,}680$ depth-$2$
  proofs. Error bars show SEM.}
  \label{fig:reconstruction_loss}
\end{figure}

\subsection{Measurement \#3: Density-Superlevel Containment}
\label{sec:topological_containment}

The third measurement stops asking for a ranking score and instead asks
whether a proof lies inside the high-density region occupied by
constructive proofs.
We estimate the constructive density $\widehat{p}_{\rm con}$ with a
Gaussian kernel density estimator fit on constructive training
embeddings, and define the $q$-th superlevel set
\[
  S_q = \{x : \widehat{p}_{\rm con}(x) \geq t_q\}
\]
as the region containing the densest $q\%$ of held-out constructive
proofs.
The threshold $t_q$ is calibrated on held-out constructive proofs, so
$10\%$ of constructive proofs fall outside $S_{90}$ by construction.
A classical population that matches the constructive support should
therefore have an outside-$S_{90}$ rate near $10\%$; a population lying
in the constructive periphery should have a substantially higher rate.

\Cref{fig:hero}(c) plots the full log-density distribution under this
constructive KDE.
Each dot is one proof, the red dashed line is the $S_{90}$ threshold,
black ticks mark bucket medians, and the percentage below each column is
the fraction of proofs outside $S_{90}$.
At depth $2$, $43\%$ of classical proofs fall outside the constructive
$S_{90}$ region, almost five times the constructive baseline.
The outside rate then declines through the depth buckets: $33\%$ at
depth $3$, $18\%$ at depths $4$--$6$, $12\%$ at depths $7$--$8$, and
$8\%$ at depth $9{+}$.
Thus direct uses of \choice{} occupy the low-density periphery of the
constructive proof distribution, while remote transitive dependents move
back into the constructive high-density region.
The full $S_{80}$, $S_{90}$, and $S_{95}$ containment sweep is reported
in \Cref{app:containment_sweep}.

\paragraph{One law to rule them all.}
The three measurements above are not independent phenomena.
They are what one would expect if each depth bucket interpolates between a constructive-like proof population and a directly-classical frontier, as described in \Cref{eq:mixture}.
The following proposition formalises this.

\begin{proposition}[Mixture law for depth-stratified measurements]
\label{prop:mixture}
Assume the depth mixture model of \Cref{eq:mixture}.
For any fixed integrable measurement function $g$,
\[
  \mathbb{E}_{x \sim Q_d}[g(x)]
  =
  (1-\lambda_d)\mathbb{E}_{x \sim P}[g(x)]
  +
  \lambda_d\mathbb{E}_{x \sim R}[g(x)] .
\]
Consequently, reconstruction loss, membership in any fixed density
superlevel set, and AUC of any fixed score against $P$ are affine
functions of the same scalar $\lambda_d$.
For AUC, take $g(y)$ to be the probability that $y$ outranks an
independent constructive proof under the score, using the standard
half-credit convention for ties.
\end{proposition}

For our three measurements, $g$ is respectively the masked-token
cross-entropy, the indicator of falling outside the constructive
$S_{90}$ region, and the ranking functional induced by the $k$-NN
anomaly score.
Thus the proposition predicts that, after calibration, the reconstruction
loss, density-containment, and AUC curves should recover the same
depth-dependent coordinate.

We test this prediction in two steps.
First, we put the three raw measurements on a common boundary-strength
axis.
For each metric, the held-out constructive baseline is set to $0$.
For AUC and containment, the strict depth-$2$ value is set to $1$;
for reconstruction, the original shallow bucket $d \leq 2$ is set to
$1$.
These anchors correspond to chance for AUC, the constructive mean
masked-token loss for reconstruction, and the $10\%$ outside-$S_{90}$
rate fixed by the density threshold for containment.
\Cref{fig:three_measurements} shows the result.
Only the two endpoints enter this calibration, so the shared descent
through the intermediate buckets is empirical rather than fit.
All three measurements recover the same depth ordering and collapse to
the constructive baseline by depth $9{+}$.

Second, we fit $\lambda_d$ from AUC alone and use the mixture law to
predict reconstruction loss and $S_{90}$ containment without using
either target measurement's depth-$3{+}$ observations.
Across the five non-anchor buckets, the resulting RMSE is $0.15$ nats
for reconstruction loss and $3.9$ percentage points for containment.
The implied weights and fit statistics are reported in
\Cref{sec:lambda-fit}.
At the level of these measurements, the depth gradient therefore obeys
a one-parameter mixture law: one scalar $\lambda_d$ per
dependency-distance bucket summarises the geometric, reconstructive,
and density-containment signatures of \choice{} dependence.

\begin{figure}[!t]
  \centering
  \includegraphics[width=0.9\columnwidth]{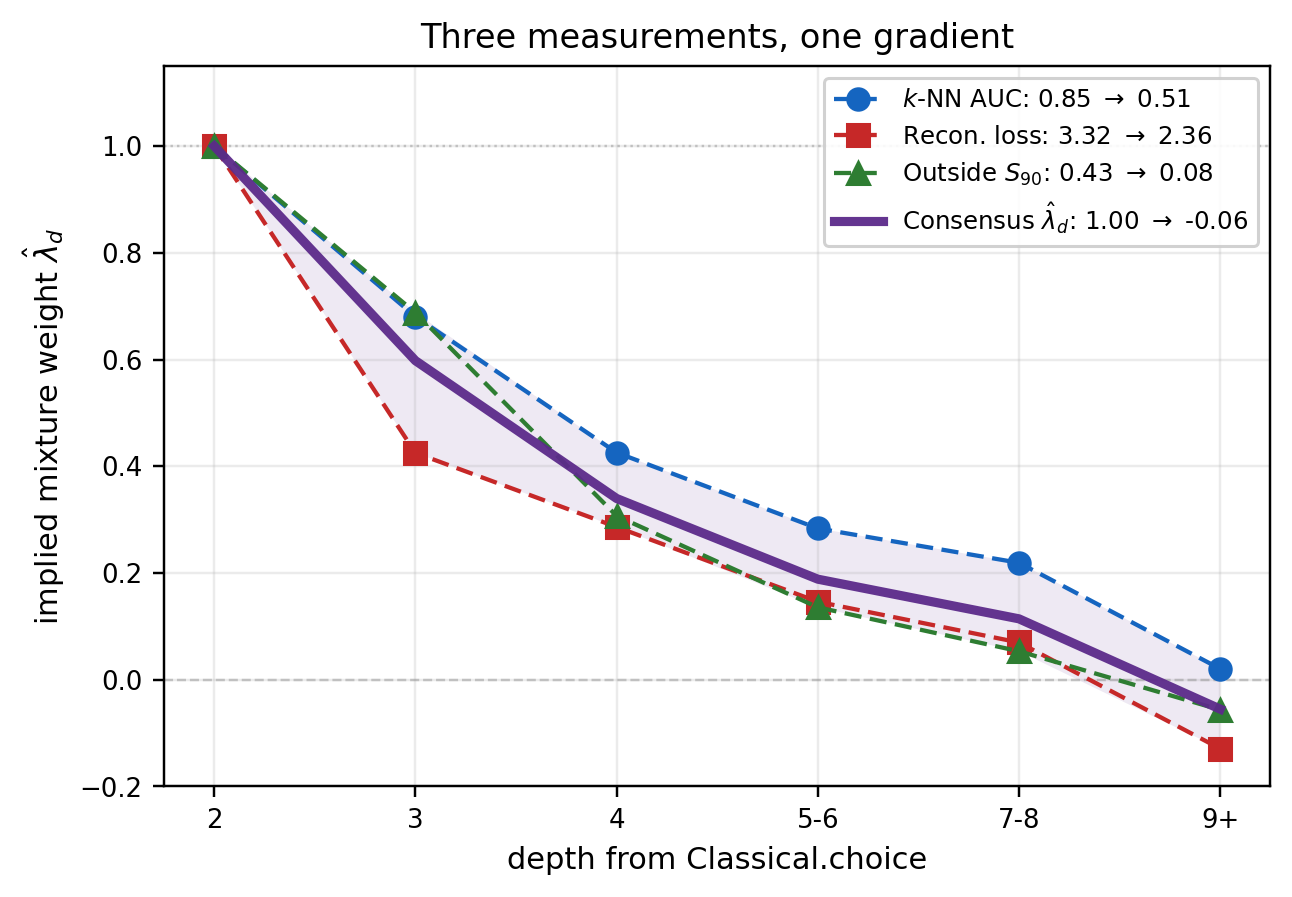}
    \caption{\textbf{Three measurements, one gradient.}
  Each metric is rescaled from the constructive baseline ($0$) to its
  shallow endpoint ($1$): strict depth $2$ for AUC and containment,
  and $d \leq 2$ for reconstruction. Intermediate buckets are not fit;
  the solid curve is the pointwise mean and the band its range. The
  separate $\lambda_d$ fit is reported in \Cref{sec:lambda-fit}.}
  \label{fig:three_measurements}
\end{figure}

\subsection{Statements Are Quiet; Proofs Are Not}
\label{sec:law:asymmetry}

The depth law could in principle be a statement-level effect rather
than a proof-level one: theorems whose proofs depend on \choice{} might
simply talk about different mathematical objects.  To test this, we
train a separate theorem-statement encoder on constructive theorem
statements only and apply the same $k$-NN protocol on theorems for which
both statement text and tactic traces are available; details are in
\Cref{app:models}.
\begin{table}[ht]
\centering
\caption{Depth-stratified $k$-NN AUC for statement and proof
embeddings on the common-coverage theorem set. The aggregate includes
the $27$ common-coverage depth-$1$ proofs, which are omitted from the
displayed depth rows.}
\label{tab:lean_statements}
\small
\begin{tabular}{lccc}
\toprule
\textbf{Bucket} & \textbf{n} & \textbf{Statement} & \textbf{Proof} \\
\midrule
depth $2$     & $3{,}556$ & $0.623$ & $\mathbf{0.847}$ \\
depth $3$     & $8{,}093$ & $0.616$ & $0.736$ \\
depth $4$     & $7{,}692$ & $0.612$ & $0.648$ \\
depth $5$     & $4{,}071$ & $0.584$ & $0.598$ \\
depth $6$     & $2{,}365$ & $0.582$ & $0.600$ \\
depth $7{-}8$ & $2{,}341$ & $0.580$ & $0.576$ \\
depth $9{+}$  & $788$     & $\mathbf{0.572}$ & $0.507$ \\
\midrule
Aggregate     & $28{,}933$ & $0.604$ & $0.675$ \\
\bottomrule
\end{tabular}
\end{table}

Statement embeddings do contain some information about the
classical/constructive partition: the aggregate AUC is $0.604$, and the
shallow buckets remain above chance.  But the signal is much weaker and
much flatter than in proofs.  Across depth, proof AUC falls from
$0.847$ to $0.507$, a swing of $0.34$, while statement AUC only experiences a swing of $0.05$. Thus theorem content carries a
modest topical signal, while the depth law itself is primarily a
proof-compositional signal.
At depth $9{+}$, where the proof trace is indistinguishable from constructive traces, the remaining separation is almost entirely statement-level.
Additional discussion can be found in \Cref{app:robustness}.

\section{Robustness}
\label{sec:robustness}

We now summarise the robustness controls. Full details are in \Cref{app:robustness}.

\paragraph{The signature is not keywords.}
Classical tactic markers (\texttt{by\_contra}, \texttt{by\_cases}, \texttt{classical}, \texttt{exfalso}, \texttt{push\_neg}, and related) appear in $67\%$ of depth-$2$ proofs and below $11\%$ at every deeper bucket.
Stripping them from the input and re-embedding with the frozen encoder drops depth-$2$ AUC from $0.847$ to $0.754$ and leaves every deeper bucket within $0.008$ of its original value (\Cref{tab:depth_ablation}, \Cref{fig:depth_ablation}).

\paragraph{The signature is not length.}
Length-residualising the encoder embeddings moves the aggregate $k$-NN, KDE, and Isolation Forest AUCs by $+0.003$, $+0.008$, and $-0.008$, respectively (\Cref{tab:lean_support_full}).
The reconstruction-loss gradient is unchanged under residualisation:
the $d \leq 2$ excess is $+0.65$ nats per masked token both raw and
residualised (\Cref{tab:reconstruction_loss}).

\paragraph{The signature is not explained by file style, author, or mathematical topic.}
Mixed-effects models with file or per-declaration \texttt{git blame} author
random intercepts, controlling for proof length and tactic statistics, yield
nearly identical classical coefficients ($+0.402$ across $5{,}386$ files;
$+0.416$ across $446$ authors).
Length-matched comparisons within file, within author, and between theorems
with closely matched statements (mean cosine $0.89$) all recover the same
shallow-to-deep attenuation.
Within-area $k$-NN over $52$ level-$2$ Mathlib subdomains likewise preserves
the boundary (median AUC $0.713$; $90\%$ above $0.6$;
\Cref{fig:confound_controls}).

\paragraph{The signature is not detector-specific.}
The sliced Wasserstein distance \citep{rabin2011wasserstein} between each depth bucket and a fixed constructive sample reaches the permutation $p$-floor at every depth, with $z$-scores declining from $65.2$ at depth $2$ to $14.5$ at depth $9{+}$ (\Cref{tab:depth_ot}).
The gradient is visible without a one-class detector.

\paragraph{The signature is not an artefact of tactic-head abstraction.}
The main encoder sees only tactic heads.
As a richer check, we trained three denoising Transformer encoders on normalised full proof source from constructive proofs only, using a BPE vocabulary learned on the constructive subset, and repeated the one-class tests.
These encoders reproduce the depth-$2$ result within $0.005$ AUC (full-source $k$-NN AUC $0.843$ after stripping classical-machinery tokens, against $0.847$ for the head-only encoder) and retain a nonzero signal at depth $9{+}$ ($k$-NN AUC $0.66$) that within-domain analyses suggest partly reflects topic and identifier content rather than proof composition.
Therefore, the head-level abstraction is not creating the shallow \choice{} boundary; rather, it removes content-level topical signal that full-source text can still exploit. Full details are in \Cref{app:full_source}.

\paragraph{The gradient generalises across three axioms, but is steepest for the axiom of choice.}
Lean~4 has three kernel-tracked axioms: \choice{}, used by classical reasoning; \texttt{propext} (propositional extensionality), used by elementary equality; and \texttt{Quot.sound}, used by quotient types.
If the depth law were a generic axiom-dependence phenomenon, all three should produce comparable gradients.
The unconditioned comparison is uninformative because the populations overlap heavily ($78\%$ of \texttt{propext}-shallow proofs are also \choice{}-shallow).
We therefore condition on \choice{}-unreached proofs and ask whether \texttt{propext} or \texttt{Quot.sound} depth still predicts anomaly within that constructive class.
\Cref{tab:axis_comparison} reports the result.
All three axioms produce depth-dependent gradients --- a proof that uses elementary equality directly is geometrically distinct from one that uses neither equality nor classical reasoning --- but \choice{} produces the steepest by a substantial margin: depth-$2$ AUC of $0.847$, versus $0.724$ for conditional \texttt{propext} and $0.665$ for conditional \texttt{Quot.sound}.
Among Lean~4's three kernel-tracked axioms, choice has the strongest geometric signature; the operational coupling with classical reasoning makes it the natural axis for the present work.

\begin{table}[ht]
\centering
\caption{Depth-stratified $k$-NN AUC under each of Lean~4's three kernel-tracked axioms.
The \texttt{propext} and \texttt{Quot.sound} columns condition on constructive: each detector trains $k$-NN on proofs unreached by both \choice{} and the listed axis, and scores buckets of constructive proofs by their depth from the listed axis.}
\label{tab:axis_comparison}
\small
\begin{tabular}{lccc}
\toprule
    \textbf{Depth} & {\small \choice{}} & {\small \texttt{propext}} & {\small \texttt{Quot.sound}} \\
\midrule
$2$       & $\mathbf{0.847}$ & $0.724$ & $0.665$ \\
$3$       & $0.736$          & $0.700$ & $0.595$ \\
$4{-}6$   & $0.625$          & $0.591$ & $0.576$ \\
$7{-}8$   & $0.576$          & $0.527$ & $0.552$ \\
$9{+}$    & $0.507$          & $0.531$ & $0.530$ \\
\bottomrule
\end{tabular}
\end{table}

The depth law is therefore not explained by theorem content alone. Its large depth-dependent component lives in proof composition.

\section{Operational Consequences for Neural Theorem Provers}
\label{sec:operational}

The depth law is a statement about proof geometry. We now ask whether the same axis is visible operationally in theorem-prover stacks, first under \texttt{aesop} alone and then under a neural-guided ReProver--\texttt{aesop} hybrid.
We sample $251$ held-out theorems across five buckets: constructive theorems and four \choice{}-reached depth ranges.
For each theorem we replace the original proof body with \texttt{by aesop} and compile with a $60$-second timeout.
A run counts as successful only if Lean exits with code $0$ and the output contains no \texttt{sorry} or \texttt{error:}.

\paragraph{Aesop Fails on Classical Theorems.} Lean's \texttt{aesop} tactic \citep{limperg2023aesop} is a best-first search over a curated rule set of lemmas tagged with the \texttt{@[aesop]} attribute.
Under this default symbolic prover, constructive and classical theorems behave very differently: \texttt{aesop} solves $20.0\%$ of constructive theorems but only $1.5\%$ of classical theorems, a Fisher odds ratio of $16.5$ ($p = 7.9 \times 10^{-6}$).
Unlike the geometric signal, this operational gap does not attenuate smoothly with dependency distance: the four \choice{}-reached depth buckets all remain near zero.
Thus proof style becomes increasingly constructive-looking with depth, while success under this prover stack remains strongly associated with whether the theorem's Mathlib proof is classical.

\paragraph{The operational gap is specific to choice.}
This effect is not a generic consequence of being close to a kernel-tracked axiom.
Lean's kernel also tracks dependence on \texttt{propext} and \texttt{Quot.sound}.
To isolate these axes from choice, we sample $60$ theorems at \texttt{propext}-depth $2$ and $60$ at \texttt{Quot.sound}-depth $2$, restricting both samples to \choice{}-unreached proofs, and run the same \texttt{aesop} pipeline.
\Cref{tab:axis_prover} reports that \texttt{aesop} solves $11.7\%$ of \texttt{propext}-shallow proofs and $13.3\%$ of \texttt{Quot.sound}-shallow proofs.
Both rates sit far above the classical $1.5\%$ (Fisher's exact $p = 1.7 \times 10^{-3}$ and $p = 4.8 \times 10^{-4}$), while neither differs significantly from the constructive $20.0\%$ in this sample ($p = 0.29$ and $p = 0.44$).
The operational gap is therefore choice-specific among Lean's kernel-tracked axioms.
\begin{table}[ht]
\centering
\caption{Operational \texttt{aesop} success rate at distance $2$ from
each of Lean~4's three kernel-tracked axioms.  \texttt{propext} and
\texttt{Quot.sound} rows are restricted to \choice{}-unreached theorems
to isolate each axis.  Wilson $95\%$ intervals.}
\label{tab:axis_prover}
\small
\begin{tabular}{lrcc}
\toprule
\textbf{Axis} & \textbf{n} & \textbf{rate} & \textbf{$95\%$ CI} \\
\midrule
\choice{}            & $47$ & $0.0\%$  & $[0.0, 7.6]$ \\
\texttt{propext}     & $60$ & $11.7\%$ & $[5.8, 22.2]$ \\
\texttt{Quot.sound}  & $60$ & $13.3\%$ & $[6.9, 24.2]$ \\
\bottomrule
\end{tabular}
\end{table}

\paragraph{Neural guidance doesn't close the gap.} \texttt{aesop}'s rule set is fixed at compile time, and its failure on classical theorems might reflect a configuration choice rather than a deeper property of the proof distribution.
We test whether neural guidance closes the gap. Results are in \Cref{tab:prover}.
Using the off-the-shelf ReProver byT5-small tactic generator of \citet{yang2023leandojo}, we read the initial proof state, beam-search the top-$8$ candidate tactics, splice each candidate followed by \texttt{all\_goals aesop} into the source, and early-exit on the first candidate that closes the proof under the same $60$-second timeout.
As a control, we also run the same hybrid pipeline using top-$8$ tactics generated for a different theorem in the same bucket.

\begin{table}[ht]
\centering
\caption{Operational success on $251$ theorems under a $60$-second
timeout. Hybrid tries ReProver's top-$8$ initial tactics followed by
\texttt{all\_goals aesop}; Shuffled uses another theorem's candidates.}
\label{tab:prover}
\small
\begin{tabular}{lrrrr}
\toprule
\textbf{Bucket} & \textbf{n} & \textbf{aesop} & \textbf{Hybrid} & \textbf{Shuffled} \\
\midrule
Constructive    & $50$  & $20.0\%$ & $\mathbf{22.0\%}$ & $10.0\%$ \\
Depth $2$       & $47$  & $0.0\%$  & $\mathbf{4.3\%}$  & $0.0\%$ \\
Depth $3{-}4$   & $55$  & $3.6\%$  & $\mathbf{5.5\%}$  & $0.0\%$ \\
Depth $5{-}6$   & $51$  & $0.0\%$  & $\mathbf{3.9\%}$  & $0.0\%$ \\
Depth $7{+}$    & $48$  & $2.1\%$  & $\mathbf{4.2\%}$  & $0.0\%$ \\
\midrule
\emph{classical combined} & $201$ & $1.5\%$ & $\mathbf{4.5\%}$ & $0.0\%$ \\
\bottomrule
\end{tabular}
\end{table}

The hybrid solves $22.0\%$ of constructive theorems and $4.5\%$ of \choice{}-reached theorems, reducing the Fisher odds ratio from $16.5$ for \texttt{aesop} alone to $6.02$ ($p = 2.9 \times 10^{-4}$).
Neural guidance therefore compresses the operational gap but does not close it.
The shuffled control shows that the improvement is theorem-specific: using tactics generated for a different theorem solves zero \choice{}-reached theorems, whereas the real hybrid solves nine.
As a pipeline diagnostic, at least one ReProver top-$8$ candidate type-checks at the initial state for $78.0\%$ of constructive theorems and $73.6\%$ of \choice{}-reached theorems, so the remaining gap is not simply a failure to emit locally valid tactics.

Finally, for \texttt{aesop}, the geometric anomaly score predicts failure beyond proof length.
A $5$-fold cross-validated logistic regression using only log proof length reaches median bootstrap AUC $0.766$ over $1000$ resamples.
Adding the $k$-NN anomaly score raises the median to $0.841$, a paired improvement of $+0.071$ AUC with bootstrap interval $[+0.003,+0.179]$ and $\mathrm{P}(\Delta > 0)=0.978$.
This effect is best interpreted at the full-sample level: within-class intervals cover zero at this sample size.

\paragraph{Scope of the operational claim.}
The gap is significant under both tested stacks.
Prefixing \texttt{aesop} with \texttt{classical;} changes no outcome across the same $251$ theorems (all McNemar $p=1.0$).
The same null holds on eight hand-written positive controls, including excluded middle and Peirce's law, even though these goals are solvable with \texttt{tauto}, and excluded middle becomes solvable by \texttt{aesop} when \texttt{Classical.em} is supplied explicitly.
Thus, in our configuration, \texttt{classical} does not expose classical primitives to \texttt{aesop}'s search rule set (\Cref{tab:classical_ablation}), so we cannot distinguish a structural barrier from a rule-set interaction.
The geometric result is independent of this interpretation, and under both tested stacks success remains strongly associated with \choice{} dependence.

\section{Discussion}
\label{sec:discussion}

In Mathlib, classical tactic traces near \choice{} occupy a distinct
region of proof space learned from constructive proofs.
This turns the classical/constructive distinction into a measured
geometry.
As dependence on \choice{} becomes more remote, that region folds back
toward the constructive baseline, and anomaly score, reconstruction loss,
and density containment recover the same descending coordinate.
This is a statement about observed proofs, the level at which theorem
provers train, retrieve, and search.

The operational experiments show that this trace has force.
Under the tested \texttt{aesop} and ReProver--\texttt{aesop} stacks,
constructive theorems close far more often than classical ones, neural
guidance narrows but does not erase the gap, and the effect is specific
to \choice{} among Lean's kernel-tracked axioms.
The depth law exposes an unexploited structural axis in learned proof systems: axiom-dependence depth predicts proof geometry and prover behaviour, yet is absent from current designs, benchmarks, and training pipelines.
How such signatures are best exploited is itself an empirical question, and we hope this work motivates the community to identify analogous structure across formal libraries.
\choice{} is not just metadata in Lean's dependency graph: it has shape, distance, and operational force.


\bibliographystyle{icml2026}
\bibliography{refs}

\newpage
\appendix
\onecolumn

\section{Reproducibility}

Code, theorem identifiers, dependency-depth labels, data splits, the
$251$-theorem operational sample, and commands to reproduce every table
and figure are available at \url{https://github.com/rodrgo/geometric-axiom-of-choice}.
All randomness is seeded, and the repository pins the software environment and commands used to regenerate the encoder checkpoints and embeddings.
The repository records the exact Lean, Mathlib, and LeanDojo revisions
used in all experiments.

\section{Limitations}
The main limitations of the study are scope limitations rather than threats to the central measurement. First, our conclusions are about observed Mathlib proofs and their tactic traces, not about an intrinsic geometry of all possible mathematical proofs. Second, our learned-representation experiments intentionally focus on tactic-mode proofs with usable traces and exclude term-mode proofs and one-step scripts such as \texttt{by rfl} or \texttt{by simp}; this yields a cleaner sequence-learning population but not a census of all declarations. Third, our main encoder uses tactic heads, which abstracts away arguments, lemma names, and local hypotheses; the full-source encoder experiments show that the shallow \texttt{Classical.choice} boundary is not created by this abstraction, but the head-level representation should still be read as a measurement of proof-compositional style rather than full semantic content. Fourth, the operational experiments measure a specific Lean 4/Mathlib prover stack, namely \texttt{aesop} and a ReProver--\texttt{aesop} hybrid, and therefore do not imply that every prover must exhibit the same success gap. Finally, our experiments identify a robust association between \texttt{Classical.choice} dependence, embedding geometry, and prover behaviour; they do not by themselves prove a unique causal mechanism for the operational gap. These qualifications are important, but they leave the main empirical conclusion intact: across label-free encoders, detector families, reconstruction loss, density-superlevel containment, length, file, author, and topic controls, full-source checks, and prover evaluations, \texttt{Classical.choice} dependence leaves a stable and measurable signature in the proof distributions used by current theorem-proving systems.

\section{Robustness}
\label{app:robustness}

\paragraph{The signature is not keywords.}

The most worrying potential confound is that the encoder simply detects explicit classical tactic tokens.
Markers such as \texttt{by\_contra}, \texttt{by\_cases}, \texttt{classical}, \texttt{exfalso}, \texttt{push\_neg}, and a small family of related keywords appear almost exclusively in classical proofs, and an encoder that learned to flag them would produce a depth-stratified AUC much like the one we report.
We rule this out by stripping every occurrence of these markers from the tactic sequences and re-embedding with the \emph{frozen} encoder.
The modification preserves $89.7\%$ of proofs unchanged and removes an average of $1.2$ tokens from each affected proof.
\Cref{tab:depth_ablation} and \Cref{fig:depth_ablation} report the result.

\begin{table}[ht]
\centering
\caption{Depth-stratified ablation of classical tactic markers. The encoder is frozen and only the input token sequences change. ``Frac.\ with marker'' is the fraction of proofs in the bucket containing at least one classical marker in the original sequence.}
\label{tab:depth_ablation}
\small
\begin{tabular}{lrccr}
\toprule
\textbf{Bucket} & \textbf{n} & \textbf{Orig AUC} & \textbf{Ablated AUC} & \textbf{Frac.\ marker} \\
\midrule
depth $2$     & $3{,}680$ & $0.847$ & $\mathbf{0.754}$ & $67.0\%$ \\
depth $3$     & $8{,}686$ & $0.736$ & $0.728$ & $10.5\%$ \\
depth $4$     & $8{,}299$ & $0.648$ & $0.646$ & $4.7\%$ \\
depth $5$     & $4{,}393$ & $0.598$ & $0.594$ & $2.7\%$ \\
depth $6$     & $2{,}580$ & $0.600$ & $0.596$ & $2.7\%$ \\
depth $7{-}8$ & $2{,}604$ & $0.576$ & $0.572$ & $2.7\%$ \\
depth $9{+}$  & $872$     & $0.507$ & $0.504$ & $2.5\%$ \\
\bottomrule
\end{tabular}
\end{table}

\begin{figure}[ht]
  \centering
  \includegraphics[width=0.98\columnwidth]{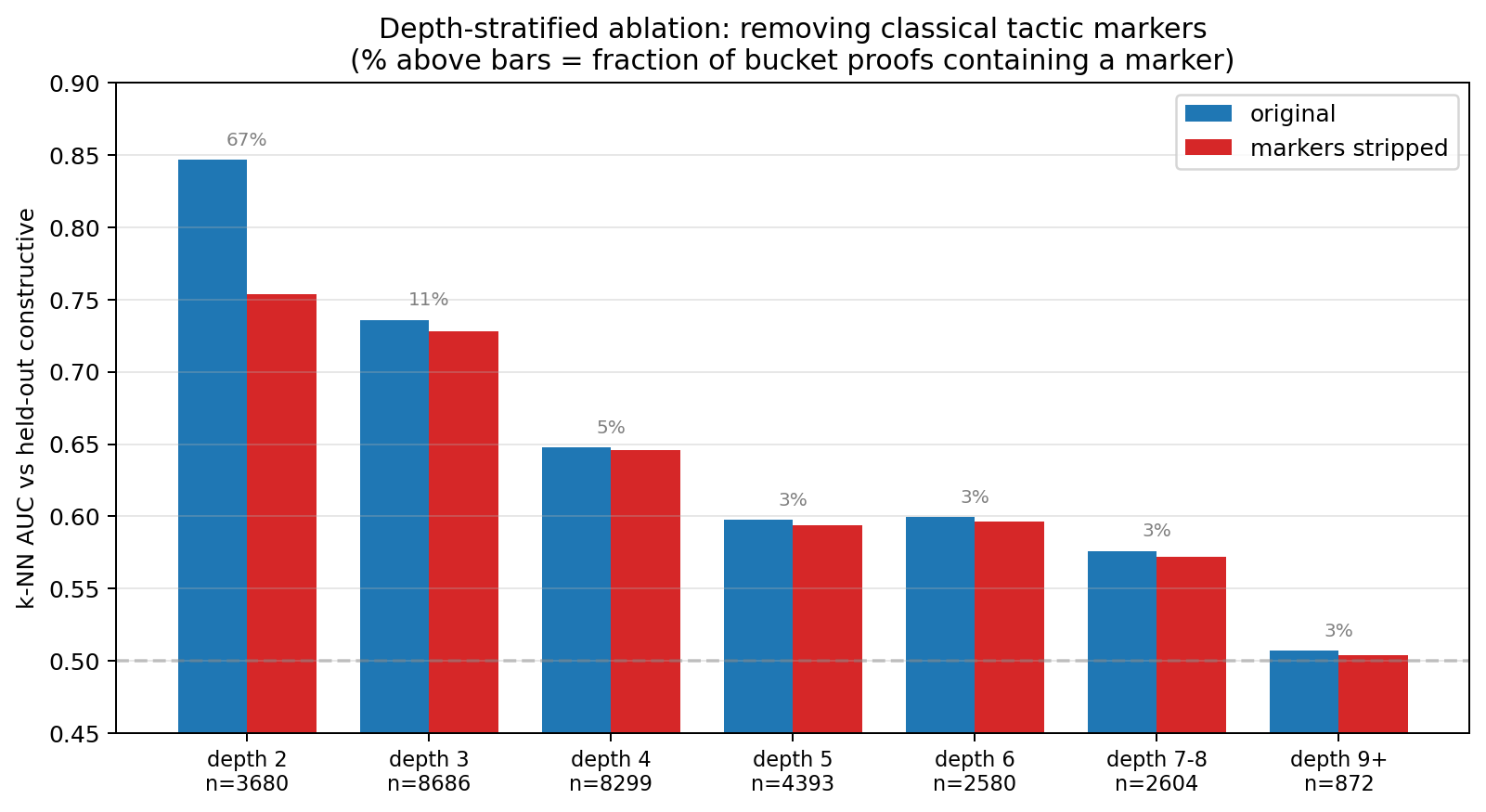}
  \caption{Depth-stratified ablation. Grey percentages mark the fraction of proofs in each bucket containing a classical tactic marker. Depth $2$ loses $0.09$ AUC (from $0.847$ to $0.754$), confirming some marker-keyword signal there. All deeper buckets lose at most $0.008$.}
  \label{fig:depth_ablation}
\end{figure}

Markers are concentrated at depth $2$, where $67\%$ of proofs contain at least one, falling below $11\%$ at every deeper bucket.
This is consistent with the depth interpretation, since proofs within one or two hops of \choice{} are precisely those that invoke classical tactics directly.
After stripping every such marker, depth-$2$ AUC is $0.754$, which is well above the $0.67$ aggregate and well separated from chance, while every deeper bucket moves by less than $0.008$.
The depth law is not a keyword artefact.
It is carried by proof structure that the encoder learned without any supervision about classical dependence.

\paragraph{The signature is not length.}

A second confound is that classical proofs may simply be longer or shorter than constructive ones, in which case any anomaly score sensitive to length would produce a similar pattern.
We control for this in two ways.
The length-residualised column of \Cref{tab:lean_support_full} fits a linear regression of each embedding dimension on proof length, computed on the constructive training set, and subtracts the prediction from every proof's embedding before scoring.
The headline $k$-NN AUC changes by $+0.003$ under this residualisation, from $0.672$ to $0.675$; the headline KDE and Isolation Forest variants change by $+0.008$ and $-0.008$, respectively, while the auxiliary one-class SVM sweep is more sensitive and is reported separately in \Cref{tab:lean_support_full}.
We apply the same residualisation to the reconstruction-loss measurement and observe the same outcome (\Cref{tab:reconstruction_loss}): the $d \leq 2$ excess of $+0.65$ nats per masked token over the residualised constructive baseline is essentially unchanged from the raw excess, and the depth-wise attenuation to depth $9{+}$ is preserved.
The depth law does not reduce to a length effect.

\begin{table}[ht]
\centering
\caption{Mean reconstruction loss (cross-entropy on
$20\%$-masked tactic tokens, averaged over $10$ maskings) by depth
bucket. ``Raw'' is the mean loss directly. ``Residualised'' is the
loss after fitting a log-length baseline on constructive training
proofs and subtracting the prediction, an excess loss over what length
alone would predict. The first classical row is the original
reconstruction bucket $d \leq 2$, containing $27$ maskable depth-$1$
proofs and all $3{,}680$ depth-$2$ proofs; all subsequent rows use the
displayed strict ranges. Significance is a Mann-Whitney $U$ test
(alternative: greater) against the held-out constructive test bucket
for the raw column, and against the residualised test bucket for the
residualised column.}
\label{tab:reconstruction_loss}
\small
\begin{tabular}{lrcc}
\toprule
\textbf{Bucket} & \textbf{n} & \textbf{Raw} & \textbf{Residualised} \\
\midrule
constructive (test) & $1{,}115$ & $2.47$ & $+0.07$ \\
depth $\leq 2$      & $3{,}707$ & $\mathbf{3.32}^{***}$ & $\mathbf{+0.65}^{***}$ \\
depth $3$           & $8{,}657$ & $2.83^{***}$          & $+0.24^{***}$         \\
depth $4$           & $8{,}265$ & $2.71^{***}$          & $+0.22^{***}$         \\
depth $5{-}6$       & $6{,}936$ & $2.59^{***}$          & $+0.14^{*}$           \\
depth $7{-}8$       & $2{,}586$ & $2.53$                & $+0.10$               \\
depth $9{+}$        & $864$     & $2.36$                & $-0.03$               \\
\bottomrule
\end{tabular}
\\[2pt]
\footnotesize{$^{*}\,p<0.05$, $^{***}\,p<10^{-8}$ (Mann-Whitney, one-sided).}
\end{table}

\paragraph{The signature is not file-level style.}

A third confound is more subtle.
Mathlib is organised into files, and different files have different authors, conventions, and tactical styles, so the geometric separation we measure might reflect file identity rather than axiom dependence.
We rule this out with a mixed-effects regression that absorbs file identity as a random intercept, with fixed-effect controls for proof length, tactic diversity, fraction of structural tactics, and fraction of automation tactics.
Across $5{,}386$ files we estimate $\widehat{\beta}_{\rm cls} = +0.402$ ($\mathrm{SE} = 0.009$, $p < 10^{-16}$): within the same Mathlib file at matched length and tactic composition, a classical proof sits $0.40$ standard deviations farther from the constructive cluster than a constructive proof.
A non-parametric companion test on $2{,}740$ length-matched within-file pairs returns a one-sided Wilcoxon signed-rank $p = 1.3 \times 10^{-117}$.
Full regression coefficients and per-bucket matched-pairs results are in \Cref{app:robustness}.

\paragraph{The signature is not the author.}
A closely related concern is that the separation reflects individual author style rather than dependence on \choice{}.
The file random intercept above already absorbs much of any author effect, because Mathlib files skew towards a single author, but a file-level author label would be a strict coarsening of that control.
We therefore attach a \emph{per-declaration} author by crediting each theorem's source span to its dominant \texttt{git blame} author in the Mathlib history at the LeanDojo trace commit, resolving authors for $99.0\%$ of the population across $446$ authors.
This label crosscuts files --- $72\%$ of authors write in more than one file and $41\%$ of files contain proofs by more than one author --- so it is not redundant with the file intercept.
Refitting the mixed-effects model of \Cref{app:robustness:file} with an author random intercept in place of the file intercept leaves the classical coefficient essentially unchanged, $\widehat{\beta}_{\rm cls} = +0.416$ ($\mathrm{SE} = 0.008$, $p < 10^{-16}$; $n = 41{,}928$), and the depth-dummy variant recovers the same gradient ($+0.64$ at depth $2$ falling to $+0.23$ at depth $9{+}$, all $p < 10^{-20}$).
A non-parametric companion that matches each classical proof to a length-matched constructive proof by the \emph{same author} ($9{,}466$ pairs, $98\%$ of them in a different file) returns a one-sided Wilcoxon gradient of $+0.86$ standard deviations at depth $2$ down to $+0.27$ at depth $7{+}$, each $p < 10^{-30}$ (\Cref{fig:confound_controls}a).
Within an author's own body of work, classical proofs remain farther from the constructive cluster; the signature is not an author-style artefact.

\paragraph{Statements are quiet, proofs are not.}

As a sanity check on \Cref{sec:law:asymmetry}, the statement encoder is not degenerate.
The frozen embeddings reach $39.7\%$ five-fold cross-validated accuracy on $24$-way Mathlib domain classification (uniform chance $4.2\%$, macro-$F_1 = 0.305$), and a ridge probe predicts $\log$ proof length with $R^2 = 0.06$.
The statement encoder captures meaningful topic and size information, but not the depth-dependent proof-composition signal measured in the main text.

\paragraph{The signature is not mathematical area.}
The within-domain analysis of \Cref{app:robustness:domain} controls for area at the level of top-level Mathlib directories, but topic could in principle be carried by proof composition at a finer granularity than the statement encoder detects.
To control for area directly, we use the theorem statement as a \emph{proof-independent} topic proxy: for each classical proof we find the constructive proof with the most similar statement embedding (greedy one-to-one matching at statement cosine $\geq 0.8$, mean $0.89$) and compare the two proofs' $k$-NN anomaly scores in the proof-embedding space.
Holding statement content approximately fixed, classical proofs remain more anomalous in proof space, and the depth law is preserved: the matched gap is $+1.38$ standard deviations at depth $2$, decaying to $+0.36$ at depth $7{+}$ and to a non-significant $+0.12$ at depth $9{+}$ ($p = 0.07$), mirroring the collapse seen elsewhere (\Cref{fig:confound_controls}a).
The result is stable across matching thresholds from cosine $0.0$ to $0.9$.
Re-running the within-area $k$-NN evaluation at Mathlib directory level~$2$ (e.g.\ \texttt{Analysis/Convex}; $52$ subdomains with at least $30$ proofs per class) preserves the boundary, with median AUC $0.713$ and $90\%$ of subdomains above $0.6$ (\Cref{fig:confound_controls}b); the weakest are classical-heavy areas such as \texttt{Geometry/Manifold}, consistent with \Cref{app:full_source}.
The depth law is therefore proof-compositional and not a reflection of mathematical area.

\begin{figure}[ht]
  \centering
  \includegraphics[width=0.98\columnwidth]{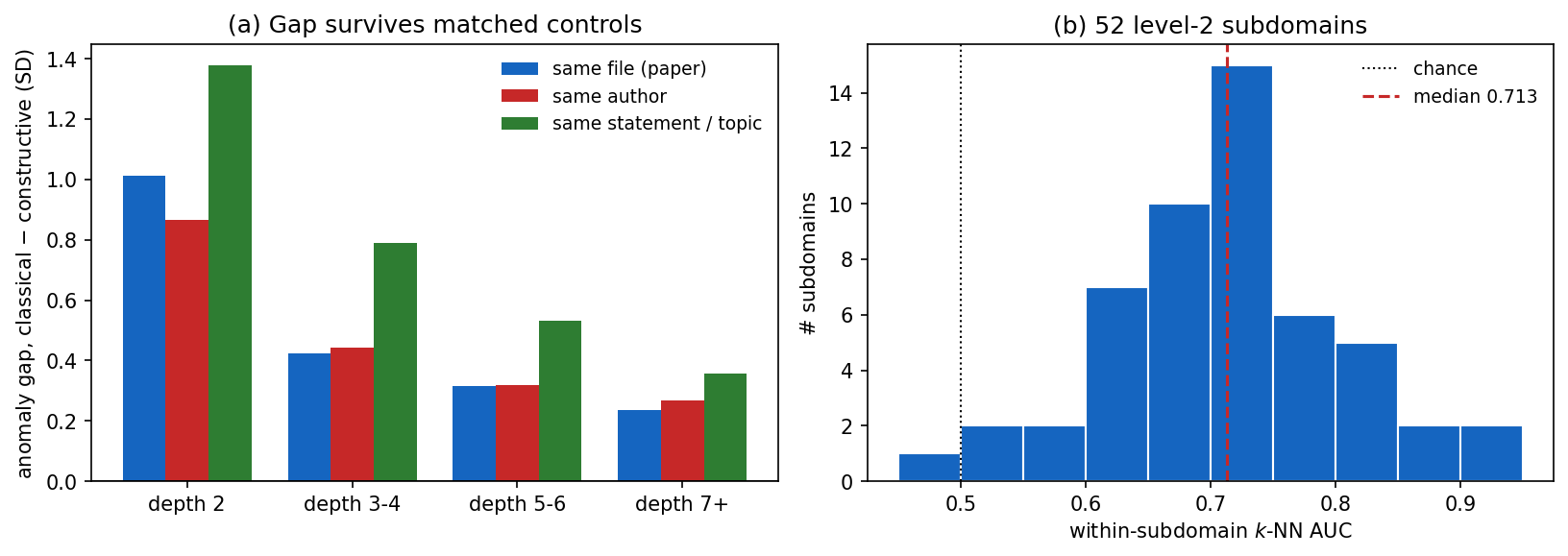}
  \caption{\textbf{The classical anomaly gap survives author and topic controls.}
  (a) Depth-stratified anomaly gap (classical $-$ constructive, in standard
  deviations) under three length-matched pairings: same file (the file-matched
  control of \Cref{app:robustness:file}), same author, and same statement
  (topic).  All three remain positive and reproduce the depth gradient.
  (b) Within-area $k$-NN AUC across the $52$ level-$2$ Mathlib subdomains with at
  least $30$ proofs per class; the boundary persists at fine topical granularity,
  with median AUC $0.713$.}
  \label{fig:confound_controls}
\end{figure}

\paragraph{Distributional confirmation.}

We close with a distributional view that does not depend on a particular detector.
The sliced Wasserstein distance \citep{rabin2011wasserstein} between two finite samples in $\reals^d$ averages one-dimensional Wasserstein distances over random projections, giving a metric sensitive to differences in any direction.
\Cref{tab:depth_ot} reports the sliced Wasserstein distance from each depth bucket to a fixed $800$-theorem constructive sample with a $1{,}000$-permutation null.

\begin{table}[ht]
\centering
\caption{Depth-stratified sliced Wasserstein distance against a fixed $800$-theorem constructive sample, $1{,}000$ permutations, $500$ projections.}
\label{tab:depth_ot}
\small
\begin{tabular}{lcccr}
\toprule
\textbf{Bucket} & \textbf{Sliced $W$} & \textbf{Null mean} & $z$ & $p$ \\
\midrule
depth $2$     & $0.0490$ & $0.0067$ & $65.2$ & $10^{-3}$ \\
depth $3$     & $0.0355$ & $0.0067$ & $44.6$ & $10^{-3}$ \\
depth $4$     & $0.0235$ & $0.0067$ & $26.3$ & $10^{-3}$ \\
depth $5{-}6$ & $0.0187$ & $0.0067$ & $18.1$ & $10^{-3}$ \\
depth $7{-}8$ & $0.0144$ & $0.0067$ & $12.2$ & $10^{-3}$ \\
depth $9{+}$  & $0.0162$ & $0.0067$ & $14.5$ & $10^{-3}$ \\
\bottomrule
\end{tabular}
\end{table}

The distributional gradient mirrors the detection-based gradient.
$z$-scores decline from $65.2$ at depth $2$ to roughly $14$ at depth $9$ and beyond, and every bucket reaches the permutation-floor $p$-value.
The depth law is visible without invoking a one-class detector at all.

\subsection{Full-Source Encoder Robustness}
\label{app:full_source}

The main proof encoder represents each proof as a sequence of tactic heads, discarding arguments, hypothesis names, and lemma citations.
This abstraction focuses representation on proof composition but raises a natural concern: perhaps the depth law is an artefact of throwing away the symbolic content of proofs.
We test this by training a second family of encoders on normalised full proof source and repeating the depth-stratified one-class tests.

\paragraph{Setup.}
For each theorem we concatenate the LeanDojo tactic strings (including tactic heads and arguments) and tokenise with a BPE vocabulary learned on constructive proofs only.
The encoder is a denoising Transformer trained only on constructive proof source, with a $512$-token context window and whole-word masked-token reconstruction at $15\%$ corruption.
We train three independent seeds and use the resulting pooled proof embeddings under the same one-class protocol as the head-level experiments.
Classical proofs are never used for encoder training, vocabulary construction, or hyperparameter selection.

We evaluate three variants of the input.
The \emph{raw} variant uses the proof source unchanged.
The \emph{stripped} variant zeroes token embeddings for atomic classical markers like \texttt{classical}, \texttt{by\_contra}, \texttt{choose}, \texttt{exfalso}, and \texttt{push\_neg}.
The \emph{combined} variant additionally removes subword pieces from multi-token classical names such as \texttt{Classical.choice}, \texttt{Classical.em}, \texttt{Classical.decEq}, and \texttt{Classical.indefiniteDescription}.
These ablations test whether the full-source signal is driven by explicit identifier leakage.

The full-source encoder learns meaningful structure but is a weaker language model than the head-level encoder.
Held-out masked-token loss is $4.25$--$4.58$ nats across seeds, above the originally chosen $3.0$-nat gate but far below the $10.31$-nat uniform-vocabulary baseline.
A frozen-domain probe reaches about $38.5\%$ accuracy on $11$ Mathlib top-level domains, and a length probe reaches $R^2 \approx 0.95$.
We therefore treat this experiment as an auxiliary robustness check, not as a replacement for the head-level encoder.

\paragraph{Depth-stratified boundary detection.}
\Cref{tab:full_source_depth} compares the full-source encoder with the head-level encoder.
The depth-$2$ boundary is essentially unchanged across encoder and stripping variants: $0.847$ head-only, $0.845$ full-source raw, $0.843$ stripped, $0.837$ combined.
The shallow boundary is not created by the tactic-head abstraction and is not removed by stripping explicit classical identifiers.

\begin{table}[ht]
\centering
\caption{Depth-stratified $k$-NN AUC for the head-level encoder and the full-source encoder.  Full-source values are medians over three encoder seeds, with the seed range in brackets.  The \emph{stripped} variant removes atomic classical-machinery tokens and the \emph{combined} variant also removes subword pieces from multi-token classical names.  The depth-$2$ boundary is stable across representations and stripping variants.}
\label{tab:full_source_depth}
\small
\begin{tabular}{lcccc}
\toprule
\textbf{Depth} & \textbf{head-only} & \textbf{full raw}
& \textbf{stripped} & \textbf{combined} \\
\midrule
$2$        & $0.847$ & $0.845$ {\scriptsize $[0.832,0.851]$}
                   & $0.843$ {\scriptsize $[0.835,0.858]$}
                   & $0.837$ {\scriptsize $[0.831,0.856]$} \\
$3$        & $0.736$ & $0.843$ {\scriptsize $[0.841,0.850]$}
                   & $0.840$ {\scriptsize $[0.840,0.850]$}
                   & $0.838$ {\scriptsize $[0.836,0.848]$} \\
$4$--$6$   & $0.625$ & $0.780$ {\scriptsize $[0.776,0.781]$}
                   & $0.779$ {\scriptsize $[0.776,0.780]$}
                   & $0.774$ {\scriptsize $[0.774,0.777]$} \\
$7$--$8$   & $0.576$ & $0.701$ {\scriptsize $[0.699,0.702]$}
                   & $0.701$ {\scriptsize $[0.698,0.701]$}
                   & $0.697$ {\scriptsize $[0.697,0.698]$} \\
$9{+}$     & $0.507$ & $0.663$ {\scriptsize $[0.663,0.665]$}
                   & $0.662$ {\scriptsize $[0.662,0.664]$}
                   & $0.661$ {\scriptsize $[0.660,0.664]$} \\
\bottomrule
\end{tabular}
\end{table}

The full-source encoder also preserves the depth trend under alternative detectors and diagnostics.
On raw full-source embeddings, depth-$2$ AUC is $0.845$ for $k$-NN, $0.710$ for Isolation Forest, and $0.689$ for KDE, with corresponding depth-$9{+}$ values $0.663$, $0.545$, and $0.619$.
Length residualisation changes every full-source $k$-NN cell by less than $0.01$.
The density-superlevel diagnostic transfers: at the constructive $S_{90}$ threshold, the full-source outside rate is about $0.39$--$0.42$ for depths $2$--$3$ and declines to $0.19$ by depth $9{+}$.
Full-source reconstruction loss moves in the same direction, with a $55\%$ excess at depth $2$ relative to held-out constructive proofs.
The shallow \choice{}-dependence is visible even when the encoder reads full proof source.

\paragraph{Why the full-source gradient is flatter.}
Unlike the head-level encoder, the full-source encoder retains a substantial signal at depth $9{+}$.
Full proof source contains lemma names, identifiers, notation, and domain-specific vocabulary, so a full-source representation can separate proofs that discuss different mathematical areas, not only proofs that use different proof-compositional moves.
We test this by repeating the $k$-NN evaluation within individual Mathlib domains, fitting and scoring the detector within each domain separately (\Cref{tab:full_source_within_domain}).

\begin{table}[ht]
\centering
\caption{Within-domain depth-stratified $k$-NN AUC for the stripped full-source encoder, reported as the median over three seeds.  Detectors are fit and evaluated within each domain.  The depth-$2$ boundary remains strong in every domain.  The deep-bucket signal is smaller in classical-heavy areas such as Analysis, suggesting that part of the all-domain full-source signal is topical or identifier-level rather than proof-compositional.  The CategoryTheory $d\geq9$ cell is omitted because the bucket has fewer than $30$ classical examples.}
\label{tab:full_source_within_domain}
\small
\begin{tabular}{lccccc}
\toprule
\textbf{Domain} & \textbf{$d=2$} & \textbf{$d=3$}
& \textbf{$d=4$--$6$} & \textbf{$d=7$--$8$} & \textbf{$d\geq9$} \\
\midrule
Algebra        & $0.833$ & $0.807$ & $0.777$ & $0.660$ & $0.690$ \\
Topology       & $0.780$ & $0.753$ & $0.721$ & $0.666$ & $0.619$ \\
RingTheory     & $0.788$ & $0.787$ & $0.745$ & $0.698$ & $0.588$ \\
CategoryTheory & $0.858$ & $0.766$ & $0.737$ & $0.722$ & --- \\
LinearAlgebra  & $0.894$ & $0.860$ & $0.793$ & $0.779$ & $0.772$ \\
Analysis       & $0.849$ & $0.823$ & $0.740$ & $0.620$ & $0.503$ \\
\midrule
all domains    & $0.843$ & $0.840$ & $0.779$ & $0.701$ & $0.662$ \\
head-only      & $0.847$ & $0.736$ & $0.625$ & $0.576$ & $0.507$ \\
\bottomrule
\end{tabular}
\end{table}

At depth $2$, the full-source boundary is robust within every domain, all above $0.78$, and stripping explicit classical names barely changes the result.
At deeper depths the two encoders answer slightly different questions.
The head-level encoder measures proof-composition style and collapses to chance by depth $9{+}$.
The full-source encoder measures both proof style and mathematical content, and can therefore keep separating deep classical proofs when their identifiers or lemma references remain unusual within the constructive training distribution.
In Analysis, where classical methods are locally common, the depth-$9{+}$ full-source AUC drops to $0.503$, essentially matching the head-level collapse.
In areas like LinearAlgebra some deep signal remains, consistent with locally unusual content rather than an additional compositional boundary.

The sharp shallow boundary near \choice{} is not an artefact of discarding tactic arguments.
It appears with almost identical AUC under a richer proof-source encoder, survives explicit stripping of classical identifiers and length residualisation, and appears across alternative one-class detectors and within-domain evaluations.
The head-level encoder is therefore not a lossy shortcut that manufactures the effect, but a controlled abstraction that removes much of the topical signal while retaining the proof-method boundary.

\subsection{The signature is steepest for the axiom of choice.}

A fourth concern is whether the geometric signature is specific to \choice{} or whether any kernel-tracked axiom would produce a similar depth gradient.
Lean~4 has three such axioms: \choice{} (used by classical reasoning), \texttt{propext} (propositional extensionality, used implicitly by elementary equality), and \texttt{Quot.sound} (used by quotient types).
If the depth law were a generic axiom-dependence phenomenon rather than something specific to choice, all three would produce comparable gradients.

We re-run BFS from \texttt{propext} and \texttt{Quot.sound} against the same kernel graph and recompute depth-stratified $k$-NN AUC under the same frozen encoder.
At first pass all three axioms produce gradients, but the populations overlap heavily.
$78\%$ of \texttt{propext}-depth-$2$ proofs in the evaluation corpus are also \choice{}-shallow (depth $2$--$6$ from \choice{}), because a proof that uses elementary equality at distance $1$ from \texttt{propext} typically also uses classical reasoning at some shallow depth from \choice{}.
This overlap makes the unconditioned comparison uninformative.

The clean test conditions on \choice{}-unreached proofs and asks whether \texttt{propext} or \texttt{Quot.sound} depth still predicts anomaly within the constructive class.
Each conditional detector trains $k$-NN on doubly-unreached proofs (no dependence on \choice{} or the listed axis) and scores depth-stratified buckets of \choice{}-unreached proofs by their depth from the listed axis.
\Cref{tab:axis_comparison} reports the result.
The \choice{} gradient is the steepest by a substantial margin: depth-$2$ AUC of $0.847$ from \Cref{tab:depth_strat} versus $0.724$ for \texttt{propext} (conditional, $n = 1{,}628$) and $0.665$ for \texttt{Quot.sound} (conditional, $n = 1{,}673$).
All three gradients converge to near-chance ($0.52$--$0.53$) at depth $9{+}$.

The depth law is therefore not unique to choice.
A \texttt{propext}-depth-$2$ proof unreached by \choice{} still looks anomalous to the encoder against doubly-unreached controls, since proofs that use elementary equality directly are geometrically distinct from proofs that use neither equality nor classical reasoning.
But \choice{} produces the steepest gradient of the three kernel-tracked axioms, exceeding the conditional \texttt{propext} AUC by $0.12$ and the conditional \texttt{Quot.sound} AUC by $0.18$ at depth $2$.
Among Lean~4's three kernel-tracked axioms, choice has the strongest geometric signature, and its operational coupling with classical reasoning makes it the natural axis for the present work.

\subsection{Optimal Transport Tests}
\label{app:robustness:ot}

As a global distributional check, we compute the earth-mover's distance and sliced Wasserstein distance between constructive and classical proof distributions on a class-balanced subsample, with significance assessed by a $5{,}000$-permutation null.

\begin{table}[ht]
\centering
\caption{Optimal transport tests of distributional separation. $n_1 = n_2 = 1{,}500$ subsample per class, sliced Wasserstein with $500$ random projections, $5{,}000$-permutation null (floor $p \approx 2{\times}10^{-4}$). $z = (\text{obs} - \text{null mean}) / \text{null std}$.}
\label{tab:ot_results}
\small
\begin{tabular}{lcccr}
\toprule
\textbf{Comparison} & \textbf{EMD} & \textbf{Sliced $W$} & $z$ & $p$ \\
\midrule
Lean: constr.\ vs.\ classical (raw)           & $0.566$ & $0.0248$ & $40.4$ & $2{\times}10^{-4}$ \\
Lean: constr.\ vs.\ classical (length-resid.) & $0.539$ & $0.0257$ & $8.5$  & $2{\times}10^{-4}$ \\
\bottomrule
\end{tabular}
\end{table}

Both comparisons reach the permutation-floor $p$-value of $2 \times 10^{-4}$, and effect sizes remain large after length residualisation ($z = 8.5$), confirming that constructive and classical proofs occupy different regions of embedding space independently of proof length.

\subsection{Within-Domain Analysis}
\label{app:robustness:domain}

If the depth law were a reflection of topical clustering across Mathlib areas, restricting the evaluation to a single area should weaken or abolish the signal.
\Cref{tab:within_domain} restricts the encoder $k$-NN evaluation to each of the ten largest Mathlib top-level directories and reports per-domain AUC.

\begin{table}[ht]
\centering
\caption{Within-domain boundary detection (top $10$ Mathlib domains by size), scored by $k$-NN ($k{=}1$) on encoder embeddings. The signal persists in every domain, with the strongest separation in MeasureTheory and the weakest in Combinatorics.}
\label{tab:within_domain}
\small
\begin{tabular}{lrrc}
\toprule
\textbf{Domain} & $n_\mathrm{con}$ & $n_\mathrm{cls}$ & \textbf{$k$-NN AUC} \\
\midrule
Analysis          & $349$  & $5549$ & $0.738$ \\
Algebra           & $1731$ & $3328$ & $0.695$ \\
Data              & $2894$ & $1622$ & $0.638$ \\
RingTheory        & $569$  & $3433$ & $0.701$ \\
Topology          & $958$  & $2271$ & $0.692$ \\
MeasureTheory     & $89$   & $2714$ & $0.853$ \\
CategoryTheory    & $723$  & $1586$ & $0.684$ \\
LinearAlgebra     & $497$  & $1646$ & $0.776$ \\
NumberTheory      & $172$  & $1596$ & $0.761$ \\
Combinatorics     & $630$  & $1115$ & $0.651$ \\
\bottomrule
\end{tabular}
\end{table}

The signal is present in every domain, with AUC ranging from $0.638$ in Data to $0.853$ in MeasureTheory.
The strongest signals appear in MeasureTheory and LinearAlgebra, consistent with the heavy use of \choice{} in those areas.
The within-domain pattern rules out domain composition as an explanation for the depth law.

\subsection{File-Matched Analysis and Mixed-Effects Regression}
\label{app:robustness:file}

Mathlib is organised into files, with different files written by different authors and following different conventions.
We control for file identity in two complementary ways.

\paragraph{Mixed-effects regression.}
Let $a_i$ be the $k$-NN ($k = 5$) anomaly score of proof $i$ in the $128$-dimensional encoder space, $z$-scored across all $N = 42{,}355$ proofs.
We fit a linear mixed model with a file-level random intercept $u_{f(i)}$ and four fixed-effect covariates: log proof length $\log(1 + n_{\rm inv})$, tactic diversity $n_{\rm distinct}/n_{\rm inv}$, the fraction of structural tactics, and the fraction of automation tactics.
The classical/constructive indicator enters as
\[
  z(a_i)
  \;=\; \alpha
       + \beta_{\rm cls}\,\mathrm{is\_classical}_i
       + X_i \gamma
       + u_{f(i)}
       + \varepsilon_i,
\]
fit by REML using \texttt{statsmodels.MixedLM}.
Across $5{,}386$ files, $\widehat{\beta}_{\rm cls} = +0.402$ ($\mathrm{SE} = 0.009$, $p < 10^{-16}$).
Replacing the binary indicator with depth dummies recovers the depth gradient inside the regression, with coefficients $+0.635$ at depth $2$, $+0.441$ at depth $3$, $+0.414$ at depth $4$, $+0.348$ at depth $5{-}6$, $+0.309$ at depth $7{-}8$, and $+0.214$ at depth $9{+}$, all significant at $p < 10^{-19}$.
The depth-$9{+}$ coefficient remains detectable in the regression even though its raw $k$-NN AUC is near chance, because once file, length, diversity, and composition are absorbed into the random and fixed effects, the residual classical-ness contribution persists.

\paragraph{File-matched pairs.}
As a non-parametric companion to the regression, we greedily pair every classical proof with the nearest constructive proof in the same file by absolute difference in $\log(1 + n_{\rm inv})$, accepting pairs only when the log-length gap is at most $0.3$ and using each constructive proof in at most one pair.
This yields $2{,}740$ matched pairs.
A one-sided Wilcoxon signed-rank test on $a_{\rm cls} - a_{\rm con}$ returns $p = 1.3 \times 10^{-117}$ overall, and per-depth mean differences are $+1.01$ at depth $2$, $+0.42$ at depth $3{-}4$, $+0.31$ at depth $5{-}6$, and $+0.24$ at depth $7{+}$, each significant at $p < 10^{-14}$.

\subsection{OOV-Robustness of the Reconstruction-Loss Result}
\label{app:robustness:oov}

The proof encoder vocabulary contains $219$ tactic-head tokens plus five special tokens, with out-of-vocabulary heads mapped to a dedicated UNK token that is excluded from masking in both training and evaluation.
UNK targets therefore cannot directly inflate reconstruction loss, but they might indirectly bias predictions on in-vocabulary targets through the context window.
To check this, we restrict the evaluation to proofs containing no UNK tokens anywhere.
OOV rates are small across all buckets, with a maximum of $2.7\%$ at depth $3{-}4$ and only $1.6\%$ at depth $2$.
The $d \leq 2$ excess reconstruction loss is $+34.5\%$ over baseline in the standard variant and $+35.3\%$ in the clean variant.
The reconstruction gap is not an OOV artefact.

\subsection{Multi-Seed Stability}
\label{app:robustness:seeds}

On the frozen encoder, $k$-NN ($k = 5$) AUC across five random train/test splits is $0.667 \pm 0.004$.
The bag-of-words ($0.603$) and hand-crafted statistics ($0.631$) baselines are deterministic and therefore have no seed variance.


\section{Full Method Sweep for Aggregate Lean Support}
\label{app:full_lean_support}

\Cref{tab:lean_support_full} presents the full hyperparameter sweep, including auxiliary SVM and LOF rows that are not used for the headline depth-law table. No classical-label AUC was consulted when fixing the starred variants reported in the main text.

\begin{table}[ht]
\caption{Full hyperparameter sweep for aggregate Lean boundary detection, using the same aggregate-sample protocol as the aggregate rows of \Cref{tab:depth_strat}. All one-class methods fit on $8{,}968$ constructive train proofs and are scored on $1{,}122$ constructive test against a fixed $5{,}000$ classical subsample drawn from the $31{,}144$ classical theorem proof-embedding population. The row marked $\star$ in each method family is our a priori choice, using conventional scikit-learn or literature defaults fixed before looking at any classical-label AUC. KDE bandwidth is selected by $5$-fold CV log-likelihood on the constructive training split over the candidate grid $\{0.3, 1.0, 3.0\}$, which is also label-free. Length-residualised column: embeddings with a linear regression on proof length (tactic-invocation count) subtracted. The SVM and LOF rows are auxiliary sweep results; the headline depth-law measurements use $k$-NN, KDE, and Isolation Forest.}
\centering
\label{tab:lean_support_full}
\small
\begin{tabular}{lcc}
\toprule
\textbf{Method} & \textbf{Raw AUC} & \textbf{Length-resid.\ AUC} \\
\midrule
$k$-NN ($k{=}1$)          & $0.674$ & $0.671$ \\
$k$-NN ($k{=}5$)$^{\star}$& $0.672$ & $0.675$ \\
$k$-NN ($k{=}10$)         & $0.668$ & $0.674$ \\
One-class SVM ($\nu{=}0.05$)       & $0.561$ & $0.626$ \\
One-class SVM ($\nu{=}0.1$)$^{\star}$& $0.572$ & $0.627$ \\
One-class SVM ($\nu{=}0.2$)        & $0.594$ & $0.626$ \\
KDE ($\mathrm{bw}{=}0.3$, CV-selected)$^{\star}$ & $0.653$ & $0.661$ \\
KDE ($\mathrm{bw}{=}1.0$) & $0.669$ & $0.675$ \\
KDE ($\mathrm{bw}{=}3.0$) & $0.661$ & $0.668$ \\
Isolation Forest$^{\star}$& $0.636$ & $0.628$ \\
LOF ($k{=}20$)$^{\star}$  & $0.584$ & $0.604$ \\
\bottomrule
\end{tabular}
\end{table}

\section{Superlevel Containment Sweep}
\label{app:containment_sweep}

\begin{table}[H]
\centering
\caption{Fraction of classical proofs inside the constructive density's superlevel set $S_q$, by depth. The constructive test baseline is $q\%$ by construction.}
\label{tab:superlevel_containment}
\small
\begin{tabular}{lrccc}
\toprule
\textbf{Bucket} & \textbf{n} & $S_{80}$ & $S_{90}$ & $S_{95}$ \\
\midrule
Constructive (baseline) & $1{,}122$ & $80.0\%$ & $90.0\%$ & $95.0\%$ \\
\midrule
depth $2$     & $3{,}680$ & $38.9\%$ & $\mathbf{57.2\%}$ & $73.1\%$ \\
depth $3$     & $8{,}686$ & $48.9\%$ & $67.4\%$ & $81.1\%$ \\
depth $4$     & $8{,}299$ & $64.3\%$ & $79.9\%$ & $89.3\%$ \\
depth $5{-}6$ & $6{,}973$ & $70.6\%$ & $85.5\%$ & $92.5\%$ \\
depth $7{-}8$ & $2{,}604$ & $75.1\%$ & $88.3\%$ & $93.7\%$ \\
depth $9{+}$  & $872$     & $81.5\%$ & $\mathbf{91.9\%}$ & $95.0\%$ \\
\bottomrule
\end{tabular}
\end{table}

\section{Fitting the Mixture Weight
\texorpdfstring{$\lambda_d$}{lambda\_d}}
\label{sec:lambda-fit}

Under \Cref{eq:mixture,prop:mixture}, each fixed measurement is affine
in $\lambda_d$. Taking held-out constructive proofs as the
$\lambda=0$ anchor, we use the strict depth-$2$ bucket as the
$\lambda=1$ anchor for AUC and containment and the original
$d \leq 2$ reconstruction bucket as the $\lambda=1$ anchor for loss.
The three measurements imply
\[
\begin{aligned}
  \widehat{\lambda}^{\mathrm{AUC}}_d
    &= \frac{\mathrm{AUC}_d-0.5}
             {\mathrm{AUC}_2-0.5},\\
  \widehat{\lambda}^{\mathrm{loss}}_d
    &= \frac{\ell_d-\ell_{\mathrm{con}}}
             {\ell_{\leq 2}-\ell_{\mathrm{con}}},\\
  \widehat{\lambda}^{S_{90}}_d
    &= \frac{0.90-c_d}
             {0.90-c_2},
\end{aligned}
\]
where $\ell_d$ is mean masked-token reconstruction loss,
$\ell_{\leq 2}$ is the mean in the original shallow reconstruction
bucket, and $c_d$ is the fraction of depth-$d$ proofs inside the
constructive $S_{90}$ superlevel set. The final expression is
equivalently the endpoint calibration of the outside-$S_{90}$ rate
plotted in \Cref{fig:three_measurements}.

\begin{table}[ht]
\centering
\caption{\textbf{Implied mixture weights.}
Each measurement independently estimates $\lambda_d$ after anchoring
held-out constructive proofs at $0$ and its shallow endpoint at $1$.
AUC and containment use strict depth $2$; reconstruction uses
$d \leq 2$. Small negative values at depth $9{+}$ are left unclipped.}
\label{tab:lambda-implied}
\small
\setlength{\tabcolsep}{4pt}
\begin{tabular}{lrrrrrr}
\toprule
& $2$ & $3$ & $4$ & $5{-}6$ & $7{-}8$ & $9{+}$ \\
\midrule
$\widehat{\lambda}^{\mathrm{AUC}}_d$
  & $1.00$ & $0.68$ & $0.43$ & $0.28$ & $0.22$ & $0.02$ \\
$\widehat{\lambda}^{\mathrm{loss}}_d$
  & $1.00$ & $0.43$ & $0.29$ & $0.15$ & $0.07$ & $-0.13$ \\
$\widehat{\lambda}^{S_{90}}_d$
  & $1.00$ & $0.69$ & $0.31$ & $0.14$ & $0.05$ & $-0.06$ \\
\bottomrule
\end{tabular}
\end{table}

As a stronger test, we fit $\lambda_d$ from AUC alone and use it to
predict the other two measurements:
\[
  \widehat{\ell}_d
  =
  (1-\widehat{\lambda}^{\mathrm{AUC}}_d)\ell_{\mathrm{con}}
  +
  \widehat{\lambda}^{\mathrm{AUC}}_d\ell_{\leq 2},
  \qquad
  \widehat{c}_d
  =
  (1-\widehat{\lambda}^{\mathrm{AUC}}_d)0.90
  +
  \widehat{\lambda}^{\mathrm{AUC}}_d c_2.
\]
Only the constructive baseline and the corresponding shallow endpoint
of each target measurement enter these predictions; no
reconstruction-loss or containment observation from depth $3$ onward
is used.
Across the five
non-anchor buckets, the resulting RMSE is $0.15$ nats for
reconstruction loss and $3.9$ percentage points for $S_{90}$
containment. Thus one mixture weight estimated from AUC reproduces the
attenuation observed in the other two measurements, supporting the
one-parameter reading of the depth law.

All estimates and fit statistics are computed from the unrounded
measurements; displayed values are rounded. For reconstruction, the
shallow anchor pools $27$ maskable depth-$1$ proofs with the $3{,}680$
depth-$2$ proofs; AUC and containment use strict depth $2$. The fit
inherits these shallow buckets as proxies for the directly-classical
frontier $R$ and tests a consequence of the mixture model, rather than
establishing that it is the unique explanation of the observed
gradients.

\section{Data, Models and Training}
\label{app:models}

All learned models are PyTorch implementations trained with AdamW (weight decay $0.01$, gradient clipping at $1.0$).
All probes and supervised baselines use scikit-learn.

\paragraph{Data.}
The full kernel dependency graph contains $471{,}260$ Mathlib declarations, of which $171{,}522$ are \choice{}-reached and $299{,}738$ are \choice{}-unreached.
For the learned proof-representation experiments, we restrict to LeanDojo-traced theorems \citep{yang2023leandojo} with a usable tactic-level proof, defined as between $2$ and $200$ tactic invocations each with an extractable leading head.
This removes term-mode proofs and one-step scripts like \texttt{by rfl} or \texttt{by simp} that carry little sequence-level information, and leaves $42{,}355$ theorems (median tactic-trace length $4$).
Full counts are in \Cref{tab:counts}.

\begin{table}[ht]
\centering
\caption{Declaration counts used in the Lean analyses.}
\label{tab:counts}
\small
\begin{tabular}{lr}
\toprule
\textbf{Population} & \textbf{Count} \\
\midrule
Declarations in Lean \texttt{Environment}   & $471{,}260$ \\
\quad Classical / Constructive              & $171{,}522$ / $299{,}738$ \\
\midrule
Theorems matched to LeanDojo                & $120{,}715$ \\
\quad Classical / Constructive              & $67{,}514$ / $53{,}201$ \\
\midrule
With usable tactic-level proofs             & $42{,}355$ \\
\quad Classical / Constructive              & $31{,}144$ / $11{,}211$ \\
\bottomrule
\end{tabular}
\end{table}

\paragraph{Lean statement denoising encoder $\phi_{\rm stmt}$.}
A Transformer encoder over subword tokens of Lean theorem statements,
trained on constructive statements only under masked-token denoising with
$20\%$ token corruption. As with $\phi_{\rm proof}$, the training loss is
token-level cross-entropy on masked positions. The frozen encoder emits
$128$-dimensional embeddings and is used only as a feature extractor for
the statement-level probes in \Cref{sec:law:asymmetry}. Thus
\Cref{tab:lean_statements} compares input views (statement text versus
tactic traces) rather than different self-supervised objectives.

\paragraph{Lean proof denoising encoder $\phi_{\rm proof}$.}
$4$-layer Transformer encoder paired with a $2$-layer Transformer decoder ($d_{\rm model} = 128$, $4$ heads, feed-forward $512$, dropout $0.1$, maximum sequence length $64$, $1.4$M parameters total).
The vocabulary consists of $219$ tactic heads occurring at least five times in constructive proofs, plus five special tokens.
The training objective replaces $20\%$ of input tokens with a \texttt{[MASK]} symbol, and the decoder is trained to recover the originals under token-level cross-entropy on the masked positions only.
Trained for $20$ epochs at learning rate $3 \times 10^{-4}$ and batch size $128$ on the $8{,}968$ constructive proofs in the training split.
The proof embedding is the mean-pooled encoder output, ignoring padded positions, and the encoder is frozen for all downstream analyses.

\paragraph{Embedding normalisation and detector feature space.}
The pooled proof embedding is $L^2$-normalised when it is saved, so the raw saved vector lies on $\mathbb{S}^{127}$. For one-class analyses, we then fit a per-coordinate \texttt{StandardScaler} on the $8{,}968$ constructive training embeddings and apply the same transform to held-out constructive and classical proofs. Thus the reported $k$-NN score is mean Euclidean distance to the $5$ nearest constructive-training proofs in this constructive-standardised feature space. The transform is fit without classical labels and is shared with KDE, one-class SVM, Isolation Forest, and LOF, whose default scores are sensitive to coordinate scale. Consequently, the reported detector scores should be read as standardised-Euclidean scores rather than cosine scores on the raw unit-sphere embeddings. As a metric ablation, cosine $k$-NN on the unstandardised $L^2$-normalised embeddings gives the same depth curve: depth-$2$ AUC is $0.841$ rather than $0.847$, and every depth bucket differs by at most $0.006$ AUC.

\paragraph{One-class detectors.}
All detectors are scikit-learn implementations fit on $\phi_{\rm proof}$ embeddings of the feature space described above.
$k$-NN uses \texttt{NearestNeighbors} with $k = 5$ as the default, with supplementary sweeps over $k \in \{1, 5, 10\}$.
Gaussian KDE uses \texttt{KernelDensity} with bandwidth selected by $5$-fold log-likelihood cross-validation on the constructive training split over the grid $\{0.3, 1.0, 3.0\}$, which picks $0.3$.
One-class SVM uses \texttt{OneClassSVM} with the RBF kernel, $\gamma = \texttt{scale}$, $\nu = 0.1$.
Isolation Forest uses \texttt{IsolationForest} with $200$ trees and contamination set to \texttt{auto}.
Local Outlier Factor uses \texttt{LocalOutlierFactor} with $k = 20$ and \texttt{novelty=True}.

\paragraph{Length residualisation.} For each proof $i \in \{1, \dots, N\}$ we have an embedding $\mathbf{e}_i \in \mathbb{R}^{128}$ and a scalar length $\ell_i \in \mathbb{R}$. We posit that each coordinate of the embedding depends linearly on length, $e_{i,j} \approx \alpha_j + \beta_j \ell_i$, and fit the intercepts and slopes by ordinary least squares on the training set, stacking them into vectors $\hat{\boldsymbol{\alpha}}, \hat{\boldsymbol{\beta}} \in \mathbb{R}^{128}$. For every proof (train and test) we then subtract the fitted prediction $\tilde{\mathbf{e}}_i \;=\; \mathbf{e}_i - \hat{\boldsymbol{\alpha}} - \hat{\boldsymbol{\beta}}\, \ell_i$, and pass the residual $\tilde{\mathbf{e}}_i$ to downstream analyses. 

\paragraph{Linear probes.}
All probes are $\ell_2$-regularised scikit-learn models with $C = 1.0$, L-BFGS solver, $2000$ iterations, evaluated by $5$-fold cross-validation.
The domain probe is multinomial \texttt{LogisticRegression} predicting Mathlib's top-level directory across the $24$ domains with at least $200$ theorems each, from frozen $\phi_{\rm stmt}$ embeddings.
The length probe is a \texttt{Ridge} regression predicting $\log(1 + n_{\rm inv})$ from the same embeddings.
The structural-feature probe is binary \texttt{LogisticRegression} on the $19$-dimensional hand-crafted statistics, trained on a class-balanced subsample.

\paragraph{Compute.}
All experiments run on a single CPU on a MacBook (Darwin, Apple silicon).
One-class detector fits take roughly one minute per dataset.
Optimal transport tests with permutation null take roughly two minutes per comparison.
The longest single operation is the \texttt{aesop} and ReProver-hybrid evaluation, which runs $251$ $60$-second Lean compiles per probe.

\paragraph{Population sizes}
We give the full description of \Cref{tab:counts} which reports the population sizes we used in this paper.
As mentioned in the text, two  filters narrow the initial set to the population we analyse in our learned-representation experiments.

First, we restrict to LeanDojo \citep{yang2023leandojo} records whose name resolves to a node in our kernel-level dependency graph, so that each theorem inherits a precise classical/constructive label; this yields $120{,}715$ theorems.
Many of these declarations are nonetheless written in term mode directly, or are auto-generated (e.g.\ structure projections, instance fields), and have no tactic script attached.

Second, for the proof-encoder analyses we keep only theorems whose LeanDojo trace contains between $2$ and $200$ tactic invocations from which we can extract a leading head identifier (the first token of each  invocation, e.g.\ \texttt{intro}, \texttt{rw}, \texttt{simp}).
This single length-based filter drops two classes of declarations at once: term-mode proofs with no tactics at all ($53,659$ theorems), and single-tactic proofs such as \texttt{by rfl} or \texttt{by simp} that carry essentially no sequence-level information ($24,701$ theorems).
No vocabulary restriction is applied at this stage---out-of-vocabulary heads are mapped to a dedicated \texttt{[UNK]} token at encoding time.  This leaves $42{,}355$ theorems, split $31{,}144$ classical and $11{,}211$ constructive, with tactic-head sequences of median length $4$, mean $5.9$, $95$th percentile $18$; the encoder pads or truncates each sequence to length $64$.
The depth-stratified analyses that follow are computed on this final set.
Note that the constructive fraction drops from $63.6\%$ in the full library to $26.5\%$ here: short utility lemmas closable by a single \texttt{simp} or \texttt{rfl} are disproportionately constructive and are filtered out by the minimum-length requirement, while classical proofs tend to be longer and survive.

\paragraph{Label-free encoder}
A central methodological commitment is that the encoder is \emph{label-free}: no information about the classical/constructive partition influences its training data, vocabulary, or hyperparameters.
Concretely, (i) the $80/10/10$ train/val/test split is carved from constructive proofs only, and classical proofs form a disjoint held-out population the encoder never sees during training, validation, or model selection; (ii) the tactic-head vocabulary is built from heads occurring at least five times in constructive proofs, so heads appearing only in classical proofs are out-of-vocabulary and map to a dedicated \texttt{[UNK]} token at encoding time; and (iii) the encoder architecture and training hyperparameters ($d_{\mathrm{model}}{=}128$, $4$-layer encoder + $2$-layer decoder, maximum sequence length $64$, mask probability $0.20$, AdamW with lr $3\!\times\!10^{-4}$, $20$ epochs) are fixed defaults set from data diagnostics (e.g., the $64$ length cap is the $95$th percentile of constructive proof length), while downstream one-class hyperparameters are either pre-registered defaults ($k$-NN $k{=}5$, OCSVM $\nu{=}0.1$, LOF $k{=}20$, Isolation Forest at scikit-learn defaults) or selected by cross-validated log-likelihood on the constructive training set alone (KDE bandwidth, PCA dimension for the superlevel analysis).
When classical proofs enter the analysis, they enter as a held-out population to be measured, never as training signal or as a selection criterion.

\section{Operational Consequences for Neural Theorem Provers}
\label{app:operational}

\subsection{Aesop Fails on Classical Theorems}
\label{sec:operational:aesop}

We now show that the depth law has direct operational consequences for current theorem provers.
Lean's \texttt{aesop} tactic \citep{limperg2023aesop} is the standard general-purpose automation in Lean~4, a best-first search over a rule set fixed at compile time.
If the depth law identifies a frontier in proof methodology, theorems on the far side of that frontier should be harder for \texttt{aesop} to solve.

\paragraph{Sampling and pipeline.}
We draw up to $60$ theorems per bucket (constructive, depth $2$, depth $3{-}4$, depth $5{-}6$, depth $7{+}$) from the encoder's held-out splits, with constructive theorems from the test and validation splits and classical theorems from the full classical population the encoder never saw.
After filtering theorems whose Mathlib source files no longer exist in our checkout, $251$ theorems remain.
For each theorem we replace the original proof body with \texttt{by aesop} and compile the modified file under a $60$-second wall-clock timeout, using the Mathlib project's compiled artefacts as dependencies.
Success requires exit code $0$ with no \texttt{sorry} or \texttt{error:} in the output.

\paragraph{Result.}
\texttt{aesop} solves $20.0\%$ of constructive theorems and $1.5\%$ of classical ones (Fisher's exact odds ratio $16.5$, $p = 7.9 \times 10^{-6}$).
Under default \texttt{aesop}, the observed classical success rate is lower by more than an order of magnitude.

\paragraph{Depth gradient within classical theorems.}
The operational gradient does not match the geometric one.
The embedding AUC of \Cref{tab:depth_strat} declines from $0.847$ at depth $2$ to $0.51$ at depth $9{+}$, while \texttt{aesop} success rates across the four classical depth buckets are $0\%$, $3.6\%$, $0\%$, $2.1\%$, essentially flat at near-zero.
The two measurements are complementary.
The embedding AUC tracks proof style, which becomes constructive-looking once classical reasoning is mediated through long lemma chains.
By contrast, \texttt{aesop} success is nearly flat across the classical depth buckets. This experiment does not determine whether that depth-insensitive barrier is structural or configuration-specific.

\paragraph{The operational gap is choice-specific.}
\Cref{sec:robustness} showed that the geometric depth gradient generalises across all three of Lean~4's kernel-tracked axioms, with \choice{} the steepest.
The operational gap behaves differently.
We sample $60$ theorems each at \texttt{propext}-distance $2$ and \texttt{Quot.sound}-distance $2$, restricted to \choice{}-unreached proofs, and run them through the same pipeline as above.
\Cref{tab:axis_prover} reports the result.
\texttt{aesop} solves $11.7\%$ of the \texttt{propext} sample and $13.3\%$ of the \texttt{Quot.sound} sample.
Both rates sit far above the classical $1.5\%$ (Fisher's exact $p = 1.7 \times 10^{-3}$ and $p = 4.8 \times 10^{-4}$), while neither differs significantly from the constructive $20.0\%$ in this sample ($p = 0.29$ and $p = 0.44$).
Proofs that depend shallowly on \texttt{propext} or \texttt{Quot.sound} but not on \choice{} look operationally constructive to \texttt{aesop}.

The two results together establish an asymmetry: all three kernel-tracked axioms produce geometric depth gradients (with \choice{} the steepest), but the $13\times$ operational gap is \choice{}-specific.
Among Lean~4's three kernel-tracked axioms, dependence on \choice{} is the one that both reshapes proof methodology measurably and is associated with the large \texttt{aesop} gap observed here.

\subsection{The Anomaly Score Predicts Aesop Failure}
\label{sec:operational:anomaly}

Beyond the binary gap, we ask whether the encoder's anomaly score
carries information about \texttt{aesop} difficulty that is not already
captured by simple proof-length statistics.  The outcome in this
subsection is \texttt{aesop}-alone failure; the neural-guided hybrid is
analysed separately in \Cref{sec:operational:hybrid}.  We fit a
$5$-fold cross-validated logistic regression predicting failure
($y = 1 - \mathrm{success}$) from log proof length and compare it to a
regression that adds the $k$-NN anomaly score from \Cref{sec:law} as a
second predictor.
For the \texttt{aesop}-alone outcome, length alone reaches a median
bootstrap AUC of $0.766$, with $95\%$ percentile interval
$[0.675, 0.837]$ over $1000$ resamples.  Adding the anomaly score raises
the median to $0.841$ ($[0.706, 0.941]$) for a paired improvement of
$+0.071$, with bootstrap interval $[+0.003, +0.179]$ and
$\mathrm{P}(\Delta > 0) = 0.978$.  The interval lies almost entirely
above zero, so the geometric signal contributes predictive information
about \texttt{aesop} failure beyond what length alone explains.  We
interpret this conservatively.  Within-group analyses are underpowered,
with only $3$ \texttt{aesop} successes among $201$ classical theorems
and $10$ among $50$ constructive theorems, and the within-class
bootstrap intervals both cover zero.  We can therefore claim that the
anomaly score carries operationally relevant information at the
full-sample level for \texttt{aesop}, but we cannot confirm from this
sample that it predicts difficulty within either class separately.

\subsection{Neural-Guided Search Compresses but Does Not Close the Gap}
\label{sec:operational:hybrid}

\texttt{aesop}'s rule set is fixed at compile time, and its failure on classical theorems might reflect a configuration choice rather than a deeper property of the proof distribution.
We test this with neural-guided search, using the ReProver byT5-small tactic generator of \citet{yang2023leandojo}.
For each of the same $251$ theorems we read the initial proof state from LeanDojo's traced corpus, beam-search the top-$8$ candidate tactics, splice each candidate followed by \texttt{all\_goals aesop} into the source, and early-exit on the first candidate that closes the proof under the same $60$-second timeout.
In effect, ReProver proposes a first move and \texttt{aesop} closes the remaining subgoals.
\Cref{tab:prover} reports the result.

\paragraph{Result.}
The hybrid solves $22.0\%$ of constructive theorems and $4.5\%$ of classical ones (Fisher's exact odds ratio $6.02$, $p = 2.9 \times 10^{-4}$).
Neural guidance helps both sides but the classical side proportionally more, compressing the gap from the $16.5\times$ odds ratio of \texttt{aesop} alone to the $6.0\times$ odds ratio of the hybrid.
The compression is substantial.
The gap remains.

\paragraph{The neural contribution is theorem-specific.}
A natural objection is that ReProver might be proposing generic opening moves that aid proof search regardless of the specific theorem.
We test this with a shuffled-tactics control that runs the hybrid using the top-$8$ tactics ReProver generated for a \emph{different} theorem in the same bucket (seed $42$, no self-pairing).
The shuffled hybrid solves zero classical theorems against the real hybrid's nine, and only $5$ constructive theorems against the real hybrid's $11$.
Generic wrong first tactics not only fail to help classical proofs but actively degrade proof search on constructive ones.
All nine classical hybrid wins occur only with the correct theorem's candidates in this control.
On constructive theorems the neural contribution splits into generic-opener effects ($4$ shared wins) and theorem-specific guidance ($7$ additional wins).

\paragraph{Paired overlap.}
The paired outcomes strongly favour the hybrid.
Across all $251$ theorems, the hybrid solves $8$ theorems that \texttt{aesop} does not and loses only $1$ (a constructive theorem) that \texttt{aesop} did solve.
All $6$ hybrid-only classical wins are genuinely new theorems for the prover stack.
Per-bucket McNemar $p$-values are not individually significant at these small discordant counts, but the direction is uniform across buckets (\Cref{tab:aesop_vs_hybrid_overlap}).

\begin{table}[ht]
\centering
\caption{Paired \texttt{aesop}--hybrid outcomes. ``both'': both proved it. ``aesop only'': \texttt{aesop} proved, hybrid did not. ``hybrid only'': vice versa. ``neither'': both failed. McNemar $p$ uses Fisher's exact on the two discordant cells.}
\label{tab:aesop_vs_hybrid_overlap}
\small
\begin{tabular}{lrrrrrr}
\toprule
\textbf{Bucket} & \textbf{both} & \textbf{aesop only} & \textbf{hybrid only} & \textbf{neither} & \textbf{n} & \textbf{McN $p$} \\
\midrule
Constructive  & $9$ & $1$ & $2$ & $38$ & $50$ & $1.00$ \\
Depth $2$     & $0$ & $0$ & $2$ & $45$ & $47$ & $0.50$ \\
Depth $3{-}4$ & $2$ & $0$ & $1$ & $52$ & $55$ & $1.00$ \\
Depth $5{-}6$ & $0$ & $0$ & $2$ & $49$ & $51$ & $0.50$ \\
Depth $7{+}$  & $1$ & $0$ & $1$ & $46$ & $48$ & $1.00$ \\
\midrule
\emph{classical total} & $3$ & $0$ & $6$ & $192$ & $201$ & --- \\
\bottomrule
\end{tabular}
\end{table}

\subsection{Scope of the Operational Claim}
\label{sec:operational:scope}

The classical-versus-constructive gap is real and significant under \texttt{aesop} alone, under the neural-guided hybrid, and across multiple statistical tests.
The gap is \choice{}-specific among Lean~4's kernel-tracked axioms.
We established this on $251$ held-out theorems with a standard prover stack.

One natural follow-up is whether the gap reduces to \texttt{aesop}'s default configuration.
A reasonable hypothesis is that classical theorems are unreachable because \texttt{aesop}'s rule set does not have classical reasoning in scope by default, and that prefixing each proof with the \texttt{classical} tactic would close most of the gap.
We test this by replacing \texttt{by aesop} with \texttt{by classical; aesop} on the same $251$ theorems and re-running the pipeline.
The ablation is a clean null: across all five buckets, zero theorems change outcome and McNemar $p = 1.0$.

Before reading the null as evidence that the operational gap is structural in proof space, we ran a positive control on eight hand-written Lean~4 goals where classical reasoning is expected to matter, including $P \lor \neg P$ and Peirce's law.
On zero of the eight does \texttt{classical; aesop} succeed where \texttt{aesop} alone fails.
The same goals are closable by \texttt{tauto}, and $P \lor \neg P$ becomes closable by \texttt{aesop} when \texttt{Classical.em} is supplied as a hypothesis explicitly, so the pipeline is working.
The \texttt{classical} tactic simply does not expose classical primitives into \texttt{aesop}'s search rule set in our Lean~4 and Mathlib configuration.
We therefore cannot distinguish from this experiment a structural proof-space barrier from an interaction with \texttt{aesop}'s default rule-set configuration.
The geometric finding of \Cref{sec:law} stands independently of either interpretation, and the operational gap stands as a measurement.
Under the tested prover stacks, success is strongly associated with \choice{} dependence at the magnitude reported.

\begin{table}[ht]
\centering
\caption{Classical-prefix ablation, paired $251$ theorems. ``Rescue'' = failed
under \texttt{aesop}, succeeded under \texttt{classical; aesop}. ``Lost'' =
succeeded under \texttt{aesop}, failed under \texttt{classical; aesop}. McNemar
$p$ is Fisher's exact on the discordant cells $(b_{01}, b_{10})$.}
\label{tab:classical_ablation}
\small
\begin{tabular}{lrrrrrc}
\toprule
\textbf{Bucket} & \textbf{n} & \textbf{aesop} & \textbf{cls;aes} & \textbf{rescue} & \textbf{lost} & \textbf{McNemar $p$} \\
\midrule
Constructive  & $50$ & $20.0\%$ & $20.0\%$ & $0$ & $0$ & $1.00$ \\
Depth $2$     & $47$ & $0.0\%$  & $0.0\%$  & $0$ & $0$ & $1.00$ \\
Depth $3{-}4$ & $55$ & $3.6\%$  & $3.6\%$  & $0$ & $0$ & $1.00$ \\
Depth $5{-}6$ & $51$ & $0.0\%$  & $0.0\%$  & $0$ & $0$ & $1.00$ \\
Depth $7{+}$  & $48$ & $2.1\%$  & $2.1\%$  & $0$ & $0$ & $1.00$ \\
\bottomrule
\end{tabular}
\end{table}

\end{document}